%% file: main.tex
\documentclass[10pt,journal]{IEEEtran}
\usepackage{graphicx}
\usepackage{mathptmx}
\usepackage{amsmath}
\usepackage{amssymb}
\usepackage{amsthm}

\newtheorem{assumption}{Assumption}
\usepackage{bbm}
\newtheorem{thm}{Theorem}
\newtheorem{lem}{Lemma}

\theoremstyle{definition}

\usepackage[utf8]{inputenc}
\usepackage{amsmath}
\usepackage{amssymb}
\usepackage{amsthm}

\usepackage{mathtools}
\usepackage{amsmath,amssymb}

\usepackage{xcolor}
\usepackage{soul}
\usepackage{multirow} 
\usepackage[noadjust]{cite}
\usepackage{citesort}
\usepackage{relsize}
\usepackage[labelsep=period]{caption}
\usepackage[left=0.6in,right=0.6in,top=0.7in]{geometry}

\usepackage{amsmath}
\usepackage{varwidth}
\usepackage{amssymb}
\usepackage{algorithm}
\usepackage{algpseudocode}
\usepackage{mathtools}
\usepackage[cmintegrals]{newtxmath}
\usepackage{array}
\usepackage[caption=false,font=normalsize,labelfont=sf,textfont=sf]{subfig}

\usepackage{fixltx2e}
\usepackage{stfloats}
\usepackage{enumitem}

\begin{document}

\title{\huge{Centralized vs. Decentralized Multi-Agent Reinforcement Learning for Enhanced Control of Electric Vehicle Charging Networks }}

% \author{\IEEEauthorblockN{1\textsuperscript{st} Amin Shojaeighadikolaei}
% \IEEEauthorblockA{\textit{EECS Department} \\
% \textit{University of Kansas}\\
% Lawrence, KS \\
% https://orcid.org/0000-0001-8638-8285}
% \and
% \IEEEauthorblockN{2\textsuperscript{nd} Morteza Hashemi}
% \IEEEauthorblockA{\textit{EECS Department} \\
% \textit{University of Kansas}\\
% Lawrence, KS \\

% \author{
%     \IEEEauthorblockN{Amin Shojaeighadikolaei\IEEEauthorrefmark{1}, Zsolt Talata\IEEEauthorrefmark{2}, Morteza Hashemi\IEEEauthorrefmark{1}} 
%     % \\
%     % \IEEEauthorblockA{\IEEEauthorrefmark{1}Department of Electrical Engineering and Computer Science, University of Kansas, Lawrence, KS, USA}
%     % \IEEEauthorblockA{\IEEEauthorrefmark{2}Department of Mathematics, University of Kansas, Lawrence, KS, USA}
% }

% \author{\IEEEauthorblockN{Amin Shojaeighadikolaei, Zsolt Talata, Morteza Hashemi}

% \IEEEauthorblockA{Department of Electrical Engineering and Computer Science, University of Kansas, Lawrence, KS, USA}

\author{Amin Shojaeighadikolaei, Zsolt Talata, Morteza Hashemi
\thanks{Amin Shojaeighadikolaei and Morteza Hashemi are with the Department of Electrical Engineering and Computer Science, and Zsolt Talata is with the Department of Mathematics at the University of Kansas, Lawrence, KS, USA (email: amin.shojaei@ku.edu, talata@ku.edu, mhashemi@ku.edu).}
% \thanks{ and  are with the Faculty of Department of Mathematics and 
% Department of Electrical Engineering and Computer Science, University of Kansas, Lawrence, KS, US. (email: ; email: mhashemi@ku.edu)}
\vspace{-0.95cm}}
% \and
% \and
% \and
% \and
\maketitle
\begin{abstract}
The widespread adoption of electric vehicles (EVs) poses several challenges to power distribution networks and smart grid infrastructure due to the possibility of significantly increasing electricity demands, especially during peak hours. Furthermore, when EVs participate in demand-side management programs, charging expenses can be reduced by using optimal charging control policies that fully utilize real-time pricing schemes. However, devising optimal charging methods and control strategies for EVs is challenging due to various stochastic and uncertain environmental factors.
Currently, most EV charging controllers operate based on a centralized model.
In this paper, we introduce a novel approach for \emph{distributed} and \emph{cooperative} charging strategy using a Multi-Agent Reinforcement Learning (MARL) framework. Our method is built upon the Deep Deterministic Policy Gradient (DDPG) algorithm for a group of EVs in a residential community, where all EVs are connected to a shared transformer. This method, referred to as CTDE-DDPG, adopts a Centralized Training Decentralized Execution (CTDE) approach to establish cooperation between agents during the training phase, while ensuring a distributed and privacy-preserving operation during execution. 
We theoretically examine the performance of centralized and decentralized critics for the DDPG-based MARL implementation and demonstrate their trade-offs. 
Furthermore, we numerically explore the efficiency, scalability, and performance of centralized and decentralized critics. Our theoretical and numerical results indicate that, despite higher policy gradient variances and training complexity, the 
CTDE-DDPG framework significantly improves charging efficiency by reducing total variation by approximately 36\% and charging cost by around 9.1\% on average. Furthermore, our results demonstrate that the centralized critic enhances the fairness and robustness of the charging control policy as the number of agents increases. These performance gains can be attributed to the cooperative training of the agents in CTDE-DDPG, which mitigates the impacts of nonstationarity in multi-agent decision-making scenarios. 
% due to the actions of other agents.  
\end{abstract}

\begin{IEEEkeywords}
Multi-agent Reinforcement Learning (MARL), EV Charging Control, Distributed and Cooperative Control. 
\vspace{-0.55cm}
\end{IEEEkeywords}

\input{Chapters/Introduction_V2}

\input{Chapters/Related_work_V1}
\input{Chapters/System_model}

\input{Chapters/algorithm}

\input{Chapters/Variance}
\input{Chapters/Numerical_Result}
\input{Chapters/Conclusion}

\bibliographystyle{IEEEtranN}
\bibliography{refs}

\end{document}

%% file: Chapters/Introduction_V2.tex
\section{Introduction}
\IEEEPARstart{T}{he} fundamental challenge in power grid management is power balancing, which is to ensure that electricity generation closely matches variable demand throughout the day. 
Electricity demand is lowest in the morning, increases in the afternoon hours, and peaks in the evening. To meet the demand, system operators constantly adjust the dispatch of various generators with different operating costs during a 24-hour cycle. 
% For instance, the off-peak hours are typically covered by baseline generators, which are relatively less costly to operate, while faster and more expensive generation facilities are used for load following and ramping during peak hours.  
As a result, the price of electricity is not constant during a day; rather, it is considerably more expensive during peak hours and grid-regulating events.
% Demand is lowest in the morning, peaks in the evening, and operators adjust generator dispatch accordingly. This leads to dynamic electricity prices, with peak hours and grid regulatory events being more expensive. 
In this context, demand-side management (DSM) programs are used to encourage consumers to shift their consumption to off-peak hours or reduce overall consumption. 
With emerging electricity loads such as electric vehicles (EVs), deploying efficient load-shifting solutions becomes even more critical, since the widespread EV adoption can significantly increase energy demand during peak hours. For example, Fig.~\ref{System Model} illustrates the EV charging network with a shared transformer power source. This network consists of two layers: (i) the physical power layer and (ii) the control layer. In the physical power layer, all EVs are connected to the upstream grid (utility company) via a shared transformer. Given dynamic pricing and underlying constraints in the physical layer (e.g., shared transformer), it is essential to develop effective management and coordination of EV chargers\footnote{The terms ``EV charger'' and ``EV'' are used interchangeably in this paper.} to manage total demand and prevent transformer overload during peak hours~\cite{nimalsiri2019survey}, as well as to minimize charging costs for EV owners~\cite{al2019review,abdullah2021reinforcement}. 

% achieve two specific goals: (i) , and (ii)
%  and satisfying different constraints in terms of arrival and departure times, charging level preference, and charging duration. 

% Thus, it is necessary to develop efficient   EV charging control algorithms to minimize costs for EV owners and meet battery needs while preventing shared transformer overload. 
% Efficient scheduling algorithms for EV charging are crucial for achieving these goals.
\begin{figure}[t]
   \centering
    \includegraphics[scale=0.304000, trim = 0.4cm 1cm 0.41cm 0.25cm, clip]{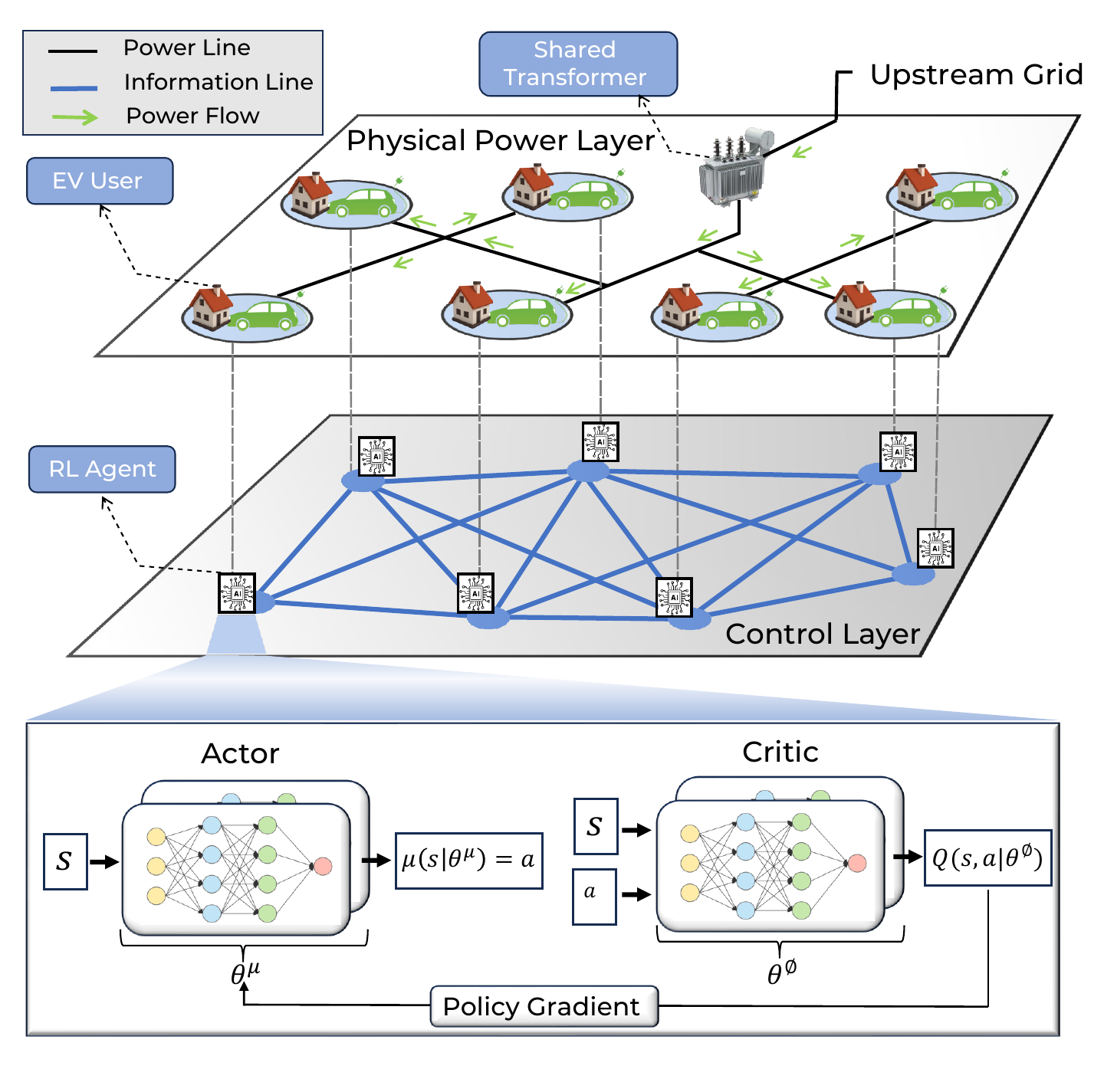}
    \caption{Electric vehicle charging network with a shared energy source.}
    \label{System Model}
    \vspace{-.18in}
\end{figure}

 % This paper focuses on exploring an optimal charging strategy for the EV network within the physical layer. 
 
 However, achieving optimal charging control faces several challenges, such as: (i) uncertainty in dynamic electricity prices throughout the day, (ii) uncertainty regarding EV owner behavior based on their arrival and departure times, charging preferences, and duration, and (iii) managing congestion and minimizing transformer overload due to the limits of the underlying physical layer, as simultaneous charging of EVs can potentially overwhelm the transformers connected to the network.
Therefore, it is desirable to develop \emph{distributed coordination and cooperation mechanisms} between EV chargers in order to react to real-time grid conditions, while ensuring optimal charging experience in terms of cost, duration, etc.

There is a multitude of prior works on model-based approaches, including binary optimization~\cite{sun2016optimal}, mixed-integer linear programming~\cite{paterakis2016coordinated}, robust optimization~\cite{ortega2014optimal}, stochastic optimization~\cite{wu2017two}, model predictive control~\cite{zheng2018online}, and dynamic programming~\cite{xu2016dynamic},   for EV charging control and optimal scheduling. These model-based methods 
% cast the EV charging control as an optimization problem but 
normally require accurate system models, which are often unavailable under uncertain conditions. In contrast, model-free approaches, such as deep reinforcement learning (DRL), do not require an accurate model or prior knowledge of the environment. Previous studies~\cite{wan2018model,zhang2023cooperative,chics2016reinforcement,9792244,9444352,zhang2020cddpg,jin2020optimal} have used single-agent DRL techniques such as Deep Q-learning (DQN), Deep Deterministic Policy Gradient (DDPG) and Soft-Actor-Critic (SAC) for an individual EV or a group of EVs. These studies assume full observability, meaning that the DRL agent has access to the local information of the EVs such as battery level and arrival/departure time. However, this assumption is not practical due to obvious
privacy and security reasons. 

To address this limitation, this paper proposes a \emph{distributed} and \emph{cooperative} strategy for EV charging control using Multi-Agent RL (MARL). We cast the problem of EV charging as a Decentralized Partially Observable Markov Decision Process (Dec-POMDP), and implement DDPG-based MARL agents on top of a residential EV network, as shown in Fig.~\ref{System Model}. 
% {Although there are several recent works focusing on distributed control using MARL for EV networks~\cite{da2019coordination,9536423,9631970,9670726,9762523,chu2022multiagent,zhang2022federated,10225344,yan2022cooperative}, 
We propose collaborative control of the EV charging network during \emph{training phase} only. In this framework, the only globally shared observation at \emph{execution phase} is the price of electricity, which is dynamically determined by the operator. 
This model departs from the assumption of sharing global or local information between agents during execution. 
% \textcolor{red}{[Lets make this part more clear and concrete.]}

To implement a MARL control strategy using DDPG agents, we explore and contrast two well-known MARL variations: decentralized critic versus centralized critic. In the former, referred to as Independent-DDPG (I-DDPG), each agent has its own critic network that is trained independently, while considering other agents as part of the environment. This independent learning offers reduced computational costs and smaller  policy gradient variances. Nevertheless, ignoring other agents' policies exacerbates nonstationarity experienced by the agents, which in turn impacts the overall learning performance and stability. Alternatively, in the centralized critic, all agents utilize a common critic network during training, while using decentralized actors during execution. This results in a centralized training-decentralized execution (CTDE) framework~\cite{lowe2017multi}, 
% DDPG for EV network, where 
which promotes collaboration between agents to mitigate nonstationarity. Nevertheless, the CTDE framework faces challenges in terms of scalability, computational complexity, and higher policy gradient variance~\cite{lyu2021contrasting,kuba2021settling}.

In this paper, we present theoretical and numerical analysis to compare the collaborative CTDE-DDPG and I-DDPG approaches for EV charging control. In particular, we theoretically illustrate that both algorithms have the same expected policy gradient, while  
% neither method consistently outperforms the other, as they both yield the same learning policy during training. Furthermore, we highlight that 
the CTDE method experiences larger variances in the policy gradient, posing a challenge to the scalability of the framework. However, the CTDE-DDPG method outperforms I-DDPG in the context of EV charging control due to the importance of agent cooperation in reacting to underlying grid conditions, such as dynamic prices, which are typically determined as a function of the total network consumption. In summary, the main contributions of this paper are as follows: 
% \vspace{-0.10in}
\begin{itemize}
    \item We formulate the problem of distributed EV charging control as an instance of Dec-POMDP, and examine two variations of MARL with decentralized and centralized critics. In the case of centralized critic, we leverage the CTDE framework to establish cooperation between agents during
the training phase, while relaxing the assumption of observing the global network parameters and exchanging private information between agents during execution. 
    % for addressing the EV network charging issue, formulating it as a POMDP with mathematical formulation for decentralized control approach.
    % \item 
    \item We theoretically analyze the performance of CTDE-DDPG and I-DDPG methods that have centralized and decentralized critic networks, respectively. We show that both methods converge to the same
expected policy gradient. 
Furthermore, the centralized critic has a larger variance in the policy gradient, which adversely affects the scalability of CTDE-DDPG.     
    % that selecting a centralized critic is not universally preferable and relies on the environment, since both CTDE-DDPG and I-DDPG methods converge to the same policy on average. In fact, employing a centralized critic results in higher variance, which adversely affects learning stability and convergence speed. 
    On the other hand, the CTDE-DDPG algorithm combats nonstationarity due to the cooperation between agents during training. 
    \item We provide a comprehensive set of numerical results for EV charging control. The results show that CTDE-DDPG outperforms I-DDPG in terms of charging total variation, charging cost, and fairness across agents. These performance gains are attributed to the cooperative behavior of the EV charging controllers to collectively respond to the electricity price signal and reduce overall consumption during peak hours, thus providing economical gains for all participants in the network.
    % in preventing transformer overload. 
    % Additionally, collaboration among agents can yield  
    The performance of CTDE-DDPG and I-DDPG for EV charging control is evaluated with up to 20 agents. 
    % , numerical results scalability, robustness, and stability
\end{itemize}

\noindent 
This paper extends our prior work in~\cite{10313469}, with two main enhancements: (i) we present theoretical results comparing the CTDE-DDPG and I-DDPG in terms of the average and variance of the policy gradient, and  
(ii) we examine the scalability, performance, and robustness of CTDE-DDPG and I-DDPG frameworks 
% for controlling an EV network 
under more realistic EV charging scenarios with up to 20 agents. 
% with up to 20 agents, which represents real-world conditions.  
% building upon preliminary results presented previously. 
% The paper also focuses on a scenario  
% All numerical evaluations are updated to include larger and more realistic EV charging scenarios. The conference paper included results for only three agents. In the submitted manuscript, we have extended the number of agents to 3, 5, 10, and 20 agents. 
%     \item In this paper, we investigate the scalability, robustness, and fairness of the CTDE-DDPG and I-DDPG algorithms as the number of agents increases.  Our extensive comparisons demonstrate that CTDE-DDPG outperforms I-DDPG due to the cooperation between agents that can collectively react to the price of electricity and reduce overall consumption during peak hours, thereby providing economic gains for all participants in the network.
The paper's structure is as follows: Section~\ref{Section:RelatedWork} reviews related works. Section~\ref{Section:SystemModel} presents the system model, followed by the algorithm's principles in Section~\ref{Section:Principal}. Section~\ref{Section:ControlStrategy} provides the MARL control strategy for the EV control problem. Numerical results are presented in Section~\ref{Section:Results} and Section~\ref{Section:Conclusion} concludes the paper.

%% file: Chapters/Related_work_V1.tex
\section{Related Work}\label{Section:RelatedWork}
\textbf{Price-aware EV charging control.} Electricity utilities have always investigated different approaches to encourage end users to participate actively in DSM programs by shifting their
consumption to off-peak hours. For instance, Time-of-Use (ToU) pricing is one of the well-known examples of a price-based DSM program. ToU represents the simplest pricing model with \emph{pre-defined} peak and off-peak time intervals, each with a tiered pricing system.
In the context of EV charging control, several researchers have presented model-based and model-free approaches for EV scheduling with ToU pricing~\cite{8910498,8345660,9672089}. 
As an extension to the ToU pricing,
% with increasing distributed energy resources (DER) and storage devices, 
real-time pricing (RTP) is more sophisticated with dynamic prices (as opposed to pre-defined structures) to balance real-time demand and load-shifting to off-peak hours~\cite{tao2020real}. In this paper, we propose an RL-based EV charging framework that is compatible with the real-time pricing scheme. In this context, there is a growing body of related work focused on model-free RL solutions for EV charging control and scheduling problems. These works fall into two categories: single-agent and multi-agent reinforcement learning methods, which are described next.

\textbf{Single-agent RL for EV charging control}. Under the assumption of complete observability of the environment, it is feasible to train a single RL agent to centrally control either an individual EV charger or a group of EV chargers. Deep Q-learning~\cite{wan2018model,zhang2023cooperative}, Bayesian Neural Networks~\cite{chics2016reinforcement}, Advantage-Actor-Critic (A2C)~\cite{9792244,9444352}, DDPG~\cite{zhang2020cddpg}, and Soft-Actor-Critic (SAC)~\cite{jin2020optimal} have been applied within this paradigm. In~\cite{zhang2023cooperative,wan2018model}, and~\cite{zhang2020cddpg} a Long Short-Term Memory (LSTM) network connected to an RL agent was used to capture the temporal uncertainty of renewable energy sources and electricity prices. However, the use of a single-agent setup for a group of EVs introduces privacy and scalability issues, especially as the number of EVs increases. To address these challenges, multi-agent RL frameworks have been proposed as a potential solution. 

% In this case, multi-agent RL frameworks have been proposed to tackle these issues. 

\textbf{Multi-agent RL for EV charging control.} Several studies have investigated EV charging control using distributed and multi-agent approaches. Qian \emph{et al.}~\cite{9536423} proposed an independent multi-agent DQN framework to learn charging pricing strategies of multiple EV stations. Similarly, Lu \emph{et al.}~\cite{9631970} leveraged multi-agent SAC for strategic charging pricing of charging station operators. These studies did not model and investigate cooperation between agents. To address this gap, the authors in~\cite{9670726} proposed an independent multi-agent SAC method using an attention layer that learns the coordination of charging behavior of EVs. In another work~\cite{9762523}, a DDPG-based MARL algorithm was proposed for EV coordination with parameter sharing using an aggregator network. Other related works~\cite{chu2022multiagent,zhang2022federated,10225344} modeled the coordination of different EV users as federated reinforcement learning, where a global aggregator network is used to handle cooperation between EV users. All the aforementioned multi-agent studies have used the parameter sharing to model the cooperation between the agents. A recent work by Yan \emph{et al.}~\cite{yan2022cooperative} introduced a cooperative MARL framework for residential EV charging. They used a neural network to approximate agent behaviors and employed the SAC method. Nevertheless, they assumed that all EV users have access to the total electricity demand at any given time, which is not feasible in realistic scenarios. 

The primary focus of this paper is on the charging control of the EV network by highlighting the importance of cooperation between agents in terms of their charging decisions. 
% i it is possible for the agents to cooperate and coordinate their actions effectively.
In contrast to~\cite{9536423,9631970}, we model the EV charging network problem by considering the cooperation between the EVs. Furthermore, compared with~\cite{9670726,9762523,chu2022multiagent,zhang2022federated,10225344}, our agents exchange information during training only, and not during the execution phase, to preserve privacy. Furthermore, unlike the method proposed in~\cite{yan2022cooperative}, our method relaxes the assumption that each EV can observe the entire network demand at any given moment. To this end, we use the MA-DDPG algorithm~\cite{lowe2017multi} as an off-policy MARL framework, which is also compatible with continuous action spaces.
% is particularly well suited for environments characterized by continuous action spaces.
Furthermore, this framework allows the DDPG agents to effectively cooperate and coordinate their EV charging actions, and thus collectively respond to dynamic grid conditions.

%% file: Chapters/System_model.tex
\vspace{-0.05in}
\section{System Model and Problem Formulation}\label{Section:SystemModel}
In this section, we present the EV charging network model and formally define the problem of optimal charging control for EV networks.

\vspace{-.5cm}
\subsection{System Model}
\textbf{EV Charging Network Model.} As shown in Fig.~\ref{System Model}, the scenario we consider includes multiple end-users who share a common energy source, such as a transformer connected to the distribution system. We model the EV network in Fig.~\ref{System Model} as a graph $\mathcal{G} = (\mathscr{V},\mathscr{E})$, where $\mathscr{V} = \{0,1,...,N\}$ and $\mathscr{E} = \{1,2,...,M\}$ represent the set of nodes (users) and edges (branches), respectively. Node zero is considered to be the connection to the shared energy source (transformer). Each user is equipped with an energy consumption scheduler (ECS) installed in their smart meter. The smart meters automatically interact  using a distributed framework to determine optimal energy consumption for EVs. To model each individual user in the network, let us define $l_i^h$ as the total consumption of the household $i$ at time $h$, where $ h \in \mathcal{H} = \{0,..., H\}$ and  $H$ denotes the last time of charging phase. Let $\mathcal{X}_i$ denote the set of appliances for user $i$. Thus, the total consumption of the individual household is obtained as follows:
\begin{equation}\label{IndividualConsumption}
l_i^h = \sum\limits_{a \in \mathcal{X}_i}{x_{i,a}^h} = \sum\limits_{a \in \mathcal{X}_i \char`\\\{EV\}}{x_{i,a}^h} + x_{i,EV}^h,
\end{equation}
where $x_{i,a}^h$ denotes the consumption of the appliance $a$ at time $h$ for user $i$.  Thus, $L_h = \sum\limits_{i \in \mathscr{V}}{l_i^h}$ represents the total network consumption at time $h$. Because our focus is on EV charging control, we assume that EV usage is the dominant term and neglect other appliance usages. 

% Since our focus is on EV charging control, without the lack of generality, we assume that EV usage is the dominant term and ignore the other appliance usages. 

% \subsection{Dynamic Pricing Tariffs Model}

\textbf{Dynamic Pricing Tariffs Model.} In electricity marketplace, electricity price is the only signal that is observable by all users through the network. In this paper, we define a function $\mathrm{F}_h(L_h)$ indicating the electricity price, which is a function of the total network consumption at time step $h$. In particular, we make the following assumption throughout this paper:
\begin{assumption}
The price function is increasing in terms of total electricity demand, such that for each $h \in \mathcal{H}$, the following inequality holds:
% \vspace{-0.13in}
\begin{equation}
    \mathrm{F}_h(\Tilde{L}_h) < \mathrm{F}_h(L_h)   \ \ \ \ \text{if} \ \ \Tilde{L}_h < L_h.
    % \vspace{-0.012in}
\end{equation}
\end{assumption}
\begin{assumption}
The price function is strictly convex. That is, for each $h \in \mathcal{H}$, any real number $L_h,\Tilde{L}_h \geq 0$, and any real number $0 \leq \theta \leq 1$, we have:
% \vspace{-0.05in}
\begin{equation}
    \mathrm{F}_h(\theta L_h + (1-\theta)\Tilde{L}_h) < \theta\mathrm{F}_h(L_h) + (1-\theta)\mathrm{F}_h(\Tilde{L}_h).
    % \vspace{-0.012in}
\end{equation}
\end{assumption}
\noindent 
An example for such an electricity price function that satisfies the  aforementioned assumptions is the quadratic function. In this paper, we consider the price function as follows:
\begin{equation}\label{Eq:Price}
    \mathrm{F}_h(L_h) = a L_h^2 + b L_h + c,
\end{equation}
where $a, b$, and $c$ are cost coefficients. The price of electricity is a function of the total demand of the network. In this paper, we assume that users do not have knowledge of the underlying price function; instead, they only have access to periodic samples of the electricity price, without any prior information about how the price function is set. Therefore, to better interact with the network and optimize the charging experience, users need to learn the price function. To this end, RL is useful to help agents learn the price function based on collected samples over time. 

\vspace{-.1cm}
\subsection{Centralized EV Network Optimization}
Our objective is to minimize the energy cost of the EV network while meeting the battery requirements of EV owners within the charging period. The primary aim of the network participants is to collaborate with each other to achieve this goal. To incentivize the participants in the cooperation task, a dynamic pricing scheme has been considered. As a result of dynamic and load-dependent pricing, we note that minimizing the charging costs could prevent transformer overload during the charging phase as well. This is because simultaneous charging of the EVs increases the charging costs, and thus an optimal charging strategy would effectively avoid this situation and prevent overheating of the transformers.  In our system model, the price signal is the only information that is broadcast to the end-users, and the users' aggregated demand is the only information sent back to the utility company.  Thus, considering~\eqref{IndividualConsumption}, we aim to minimize the network total cost subject to the constraints on the EV battery charge and arrival/departure times. Therefore, we define the EV charging control problem as follows:
% \begin{align}\label{network_cost}
% \vspace{-0.1in}
% & \min_{l_i^h} \sum\limits_{h=1}^{H} \mathrm{C}_h( \sum\limits_{i \in \mathcal{V}} l_i^h),  \label{network_cost} \\
% & \text{subject to:} \nonumber \\
% &\quad B_i(h+\Delta h) = B_i(h) + \eta \times l_i^h \times \Delta h, \label{eq:con:a1} \\
% &\quad 0 \leq B_i(h) \leq B_i^{\max}, \label{eq:con:a2} \\
% &\quad 0 \leq l_i^h \leq l_i^{{\max},h},  \label{eq:con:a3} \\
% &\quad 0 \leq h_i^{\text{arr}} < h_i^{\text{dep}} \leq H, \label{eq:con:a4}
% \end{align}
\begin{align}
\min_{l_i^h} \quad & \sum_{h=1}^{H} \mathrm{C}_h\left(\sum_{i \in \mathcal{V}} l_i^h\right) \label{network_cost} \\
\textrm{s.t.} \quad & B_i(h+\Delta h) = B_i(h) + \eta \times l_i^h \times \Delta h, \label{eq:con:a1} \\
& 0 \leq B_i(h) \leq B_i^{\max}, \label{eq:con:a2} \\
& 0 \leq l_i^h \leq l_i^{\max,h}, \label{eq:con:a3} \\
& 0 \leq h_i^{\text{arr}} < h_i^{\text{dep}} \leq H. \label{eq:con:a4}
\end{align}
% \begin{align}
%  \textbf{User level:} 
%  \begin{cases}
%  \mathop{\mathrm{minimize}} & \limits_{l_i^h}  \sum_{h=1}^{H} \mathrm{C}_h\left(\sum_{i \in \mathcal{V}} l_i^h\right) \label{network_cost} \\ 
%     \text{subject to:}  \quad & B_i(h+\Delta h) = B_i(h) + \eta \times l_i^h \times \Delta h, \label{eq:con:a1}  & \\
%      0 \leq B_i(h) \leq B_i^{\max}, \label{eq:con:a2} \\
% & 0 \leq l_i^h \leq l_i^{\max,h}, \label{eq:con:a3} \\
% & 0 \leq h_i^{\text{arr}} < h_i^{\text{dep}} \leq H. \label{eq:con:a4}
% \label{eq:UserLevel}
% \end{cases}
% \end{align}
where $\mathrm{C}_h$ in \eqref{network_cost} denotes the electricity cost function at time step $h$, which is obtained as the total demand multiplied by the unit price, i.e., $\mathrm{C}_h = L_h F_h(L_h)$~\cite{mohsenian2010autonomous}. Constraints~\eqref{eq:con:a1} and \eqref{eq:con:a2} relate to the EV battery model in which $\eta$ denotes the charging efficiency factor; $B_i(h)$ is the battery state-of-charge at time $h$; and $B_i^{\max}$ denotes the EV battery capacity for EV user $i$. In \eqref{eq:con:a3}, we impose a maximum power consumption for user $i$ at time step $h$. 
% $l_i^{max,h}$ indicates the maximum total consumption of agent $i$ at time step $h$.
Additionally, $h_i^{\text{arr}}$ and $h_i^{\text{dep}}$ in \eqref{eq:con:a4}, respectively, represent the arrival and departure times of the $i^{th}$ EV.

% \begin{equation}
% \label{eqn:op2}
%     \begin{cases}
% \mathop{\mathrm{minimize}}\limits_{\{X_i\}} & \sum_{i \in \mathcal{C}} X_i \\
% \text{subject to:} &  r_{ij} u_{ij} \geq R X_i, \ \forall i \in \mathcal{C}, \forall j \in \mathcal{D},  \\
% &  \sum_{i=1 }^{C} u_{ij}  = 1, \ \forall j \in \mathcal{D}, \\
% & u_{ij} = 1 - I\{X_i = 0\}, \forall i \in \mathcal{C}, \forall j \in \mathcal{D}. 
% \end{cases}
% \end{equation}

This optimization problem can be solved in a centralized fashion using convex optimization techniques such as the interior point method (IPM)~\cite{mohsenian2010autonomous}. However, doing so necessitates a centralized controller with access to all users' data, which introduces scalability, privacy, and security issues. Therefore, it is desirable to devise distributed control policies that can be implemented in each smart meter for its charge control functionality, with the least amount of information exchange with the energy source and other smart meters.

\vspace{-0.1in}
\subsection{Distributed EV Network Optimization}
To solve the problem in~\eqref{network_cost} in a decentralized fashion, we need to define the total network optimization problem from an individual user perspective. To do this, we define $b_i^{h}$ as the charging cost of user $i$ at time $h$. At any given time, users are charged proportional to their total energy demand. This means:
\begin{equation}\label{ProportionalDemand}
    \frac{b_{i}^h}{b_{m}^h} = \frac{l_{i}^h}{l_{m}^h} \ \ \ \forall i,m \in \mathscr{V}.
\end{equation}
By using~\eqref{ProportionalDemand},   the total cost of the network from the $m^{th}$ user's perspective at time $h$ is given by:
\begin{equation}\label{extended_proportionalDemand}
    \sum\limits_{i \in \mathscr{V}}{b_{i}^h} = \sum\limits_{i \in \mathscr{V}}{\frac{{b_{m}^h} \times l_{i}^h}{l_{m}^h}} = \frac{b_{m}^h}{l_{m}^h} \sum\limits_{i \in {\mathscr{V}}}{l_{i}^h}.
\end{equation}
Together from \eqref{network_cost}, \eqref{ProportionalDemand}, and \eqref{extended_proportionalDemand} for each user we have:
\begin{align}\label{billing}
    b_{m}^h = \frac{l_{m}^h}{\sum\limits_{i \in \mathscr{V}}{l_{i}^h}}{\sum\limits_{i \in \mathscr{V}}{b_{i}^h}} = \frac{\kappa \times {l_{m}^h}}{\sum\limits_{i \in \mathscr{V}}{l_{i}^h}}\mathrm{C}_h\left( \sum\limits_{i \in \mathscr{V}}{l_i^h}\right) =
    \nonumber \\
    \frac{\kappa \times {l_{m}^h}}{\sum\limits_{i \in \mathscr{V}}{l_{i}^h}} \mathrm{C}_h\left( {l_m^h} + \sum\limits_{i \in \mathscr{V} \char`\\\{m\}}{l_{i}^h}\right),
\end{align}
where $\kappa$ is a constant coefficient. Equation~\eqref{billing} illustrates that at any given time, the cost of user $m$ depends not only on its local consumption $l_m^h$, but also on the total consumption of other users given by $\sum_{i \in \mathscr{V} \char`\\\{m\}}{l_{i}^h}$. 
Therefore,  each agent $m$ aims to minimize its cost function by adjusting its charging power $l_m^h$ defined as follows:  
\vspace{-0.1in}
\begin{align}\label{NetwrokOptimization}
&\min_{l_{m}^h} \quad \frac{\kappa \times l_{m}^h}{\sum\limits_{i \in \mathscr{V}} l_{i}^h}  \mathrm{C}_h\left(l_{m}^h + \sum\limits_{i \in \mathscr{V} \setminus\{m\}} l_{i}^h \right).  
% \nonumber \\
% & \text{subject to:} \nonumber \\
% &\quad B_i(h+\Delta h) = B_i(h) + \eta \times l_i^h \times \Delta h, \nonumber \\
% &\quad 0 \leq B_i(h) \leq B_i^{max}, \nonumber \\
% &\quad 0 \leq l_i^h \leq l_i^{{max},h}, \nonumber \\
% &\quad 0 \leq h_i^{arr} < h_i^{dep} \leq \mathscr{H}, 
\end{align}
% However, this objective is subject to various constraints that need to be taken into account. These constraints include \textbf{network-related factors} such as power balance, \textbf{physical EV limitations} like battery charging limits and rates, and battery maintenance considerations. Additionally, \textbf{EV owners' constraints}, such as arrival and departure times, expected battery level at departure, and remaining battery energy at arrival time, should also be considered.
The optimization objective in~\eqref{NetwrokOptimization} represents the total network cost from a single end-user's perspective at time $h$. Considering the corresponding constraints, the user $i$ can solve the problem in~\eqref{NetwrokOptimization} as long as it knows the total EV consumption of other users, without requiring detailed information about the consumption of each individual EV within the network. 
% Knowing the total EV consumption of the network is sufficient. 
This problem has been solved in~\cite{mohsenian2010autonomous} using game-theory with two assumptions: (1) End-users are charged in proportion to their energy usage, independently of their usage time. This assumption is not compatible with the dynamic and real-time pricing method, by which the charging cost also depends on the time of use. (2) The daily energy consumption of all appliances, including EV, should be predetermined. The authors in~\cite{yan2022cooperative} solved this problem by relaxing these two assumptions but assumed that each EV is capable of observing the total demand of the network at any given time.
% which introduces some privacy concerns.
However, this information cannot be obtained by the users in real world scenarios. 
Hence, we pose this question that \emph{how can user $i$ solve the problem defined in~\eqref{NetwrokOptimization} locally without knowing about other users' EV consumption?} To address this, next we present an algorithm that establishes a distributed solution for the EV network control.

%% file: Chapters/algorithm.tex
\section{Principles of the algorithm}\label{Section:Principal}
In this section, we present the principles of the algorithm required to develop a distributed charging control for EV networks. To this end, first we review the foundations of the Policy Gradient method, and then present the agent setup for our formulated problem.  

\vspace{-.3cm}
\subsection{Policy Gradient Method}\label{PolicyGradientTheorem}
In contrast to Q-learning, which involves the learning of a Q-function to subsequently derive a policy by maximizing the Q-function within a given state, the policy gradient method directly optimizes an agent's policy $\pi$ that is parameterized by $\theta^\pi$. The core concept of the policy gradient method revolves around adjusting the policy parameter $\theta$ in the direction of the gradient $\nabla_{\theta} J(\theta^\pi)$ in order to maximize $J(\theta) = \mathbb{E}_{a \sim \pi}[R]$, 
% where $J(\theta^\pi) = \mathbb{E}_{a \sim \pi}[R]$,
where $a$ and $R$ are the action and reward terms, respectively. According to the policy gradient theorem~\cite{sutton1999policy,silver2014deterministic}, the gradient is computed as follows: 
\begin{align}\label{Policy-Gradient}
    \nabla_{\theta} J(\theta^\pi) = \int_{\mathcal{S}} \rho^\pi(s) \int_{\mathcal{A}} \nabla_{\theta}\pi(a|s)Q^\pi (s,a)dads = \nonumber \\ \mathbb{E}_{s \sim \rho^\pi,a \sim \pi}[\nabla_{\theta}\log \pi(a|s)Q^\pi (s,a)],
\end{align}
where $\rho^\pi(s)$ denotes the state distribution that does not depend on the policy parameters. 
One of the technical challenges is how to estimate the action-value function $Q^\pi (s,a)$.
% where $Q^\pi(s,a)$ is parameterized by $\theta^\phi$. 
One simple approach is to use a sample return to estimate the value of $Q^\pi (s,a)$, which leads to a variant of the REINFORCE algorithm~\cite{williams1992simple}. 
% The policy gradient algorithms are based on the idea that directly adjust the policy parameter $\theta$ in the direction of $\nabla_{\theta}J(\theta)$ in order to maximize $J(\theta) = \mathbb{E}_{a \sim \pi}[R]$.

\textbf{Deep Deterministic Policy Gradient (DDPG)} is an extension of the policy gradient framework with deterministic policy $\mu$. Note that for the notation clarity, we use $\mu$ to denote deterministic policies. DDPG consists of two networks: Actor and Critic. The term deterministic refers to the fact that the actor network outputs the exact action instead of the probability distribution over the actions, that is, we have $\mu(s) = \arg\max_a{Q(s,a)}$. At any given time $h$, the parameterized actor function $\mu(a\mid s)$, with parameter $\theta^\mu$, represents the policy that deterministically maps states to specific actions. In addition, the critic network describes the action-value function $Q^\mu(s,a)$ parameterized by $\theta^{\phi}$. Similar to Deep Q-learning (DQN), DDPG also employs a target network and operates as an off-policy algorithm, gathering sample trajectories from an experience replay buffer. The experience replay buffer $\mathcal{D}$ contains the tuple $\langle s, a, r, s^\prime \rangle$ and the action-value function $Q^\mu(s,a)$ is updated as:
\begin{equation}
    \mathcal{L}(\theta) = \mathbb{E}_{s,a,r,s^\prime}[(Q^\mu(s,a)-y)^2],
\end{equation}
where $y=r + \gamma Q^{\mu^\prime}(s^\prime, a^\prime)|_{a^\prime = \mu^\prime(s^\prime)}$, and $\mu^\prime$ is the target policy.

\vspace{-.3cm}
\subsection{Agent Setup for EV Network}
In our system model presented in Fig.~\ref{System Model}, each RL agent interacts with the environment such that its goal is to collect the maximum reward possible from the environment through its actions. This scenario can be modeled as a decentralized partially observable Markov decision process (Dec-POMDP), which is an extension of an MDP process into decentralized multi-agent settings with partially observability. A Dec-POMDP is formally defined by the tuple $ \langle \mathcal{I}, \mathcal{S}, \mathcal{A}, \mathcal{O}, \mathcal{T}, \mathcal{R}, \gamma \rangle$, in which $\mathcal{V}$ is the set of agents, $\mathcal{S}$ is the set of states, $\mathcal{A}$ is the joint action set, $\mathcal{O}$ is the joint observation set, $\mathcal{T}: \mathcal{S}\times \mathcal{A} \times \mathcal{S} \rightarrow [0,1]$ is the transition probability function, $\mathcal{R}:\mathcal{S} \times \mathcal{A} \rightarrow \mathbb{R}$ is the reward function set, and $\gamma \in (0,1)$ is the discount factor.
% \begin{itemize}
%     \item $\mathcal{I}$ is the set of agents,
%     \item $\mathcal{S}$ is the set of states,
%     \item $\mathcal{A}$ is the joint action set,
%     \item $\mathcal{O}$ is the joint observation set,
%     \item $\mathcal{T}: \mathcal{S} \times \mathcal{A} \times \mathcal{S} \rightarrow [0,1]$ is the transition probability function,
%     \item $\mathcal{R}:\mathcal{S} \times \mathcal{A} \rightarrow \mathbb{R}$ is the reward function set,
%     \item $\gamma \in (0,1)$ is the discount factor.
% \end{itemize}

Dec-POMDPs represent a sequential decision-making framework that extends single-agent scenarios by considering joint observations and joint actions across multiple agents. At each time step, a joint action $\mathbf{a} = <a_h^1, a_h^2, ..., a_h^{|\mathcal{I}|}>$, $\mathbf{a} \in \mathcal{A}$, is taken. In Dec-POMDPs each agent knows its own individual action, but there is no information of other agents' actions. Furthermore, each agent is only able to observe a subset of the environment states due to various factors such as physical limitations, and data privacy and security. After taking an action, each agent receives its corresponding immediate reward $r_h^i, i \in \mathcal{I}$. For the EV charging network, we define each agent's action set, observation set, and reward function as follows:

\textbf{Action Set}: Each agent $i$ has a continuous action set $\mathcal{A}_i = \{a_i: 0\leq a_i \leq a^{\max}, a^{\max} > 0 \}$. The continuous action represents the charging power. The EV battery level is calculated as $B_i(h+1) = B_i(h) + \eta \times a_i^h \times \Delta h$, where $B_i(h)$ is the battery level at time $h$, $\eta$ is the battery efficiency, $a_i^h$ is the charging power in $kW$, and $\Delta h$ is the charging period over which the charging power remains constant. 
% The joint action set is defined as $\mathcal{A} = \times_{i \in \mathcal{I}}\mathcal{A}_i$.

\textbf{Observation Set}: 
% At any given time $h$, each agent $i$ has a partial observability of the environment. 
During the training phase, the observation set for each agent $i$ is defined as
$o_i = \{ \Delta B_i^h, \Delta h_i, \mathrm{F}_h, Pl_i, h_i^{\text{dep}} \}$, where $\Delta B_i^h = B_i^{\text{exp}} - B_i(h)$ is the difference between the desired battery level and the current battery level at time step $h$. Furthermore, $\Delta h_i = h - h_i^{\text{arr}}$ represents the difference between the arrival time $h_i^{\text{arr}}$ and the current time step, and $\mathrm{F}_h$ is the electricity price at time step $h$. $Pl_i$ is a binary flag such that $Pl_i=0$ represents the EV $i$ is not connected to the charging network and $Pl_i=1$ otherwise.
% represents that the corresponding EV is plugged-in. 
Furthermore, $h_i^{\text{dep}}$ denotes the departure time of EV $i$. 

\textbf{Reward Function}: 
In MA-DDPG, each agent has its own reward function $r_i^h$,  
% which is based on the local observation of the agent. In this case,  
which represents the immediate reward for the agent $i$ at time $h$ that is obtained by taking action $a_i^h$ and  the state transition from $s_i^h$ to $s_i^{h+1}$. According to the objective of user satisfaction and network requirements, we define the reward function as follows:
\begin{equation}
    r_i^h = - \alpha_1 \times \mathrm{F}_h \times a_i^h - \alpha_2 \times (\Delta B_i^h)^2 + \mathcal{E} \times \mathbbm{1}\{\Delta B_i^\text{dep} > \sigma\},
\end{equation}
where $\alpha_1$ and $\alpha_2$ are constant coefficients, and $\mathcal{E}$ is a penalty term to provide a large negative reward based on the distance from the expected battery level, such that if $\Delta B_i^\text{dep}$ is larger than the threshold $\sigma$ (that is set based on the charging preference of the user), the agent is penalized by $\mathcal{E}$.
% \blue{if the agent fails to fully charge the EV before the departure time.}

%% file: Chapters/Variance.tex
\vspace{-.2cm}
\section{Multi-Agent Control Strategy}\label{Section:ControlStrategy}
In this section, we first examine two variations of the MARL methods for decentralized EV charging control as formulated in ~\eqref{NetwrokOptimization}. 
% decentralized control solution for~\eqref{NetwrokOptimization} within the proposed EV charging network. 
These two variants are recognized as the centralized critic and the decentralized critic. 
Next, we present a theoretical analysis to explore the convergence and performance of these two variants in order to highlight their advantages and disadvantages in the learning process.

\vspace{-.3cm}
\subsection{Multi-Agent Methods}
Consider the EV network with ${N}$ agents that have deterministic policies 
$\boldsymbol{\mu}=\{ \mu_1, ..., \mu_N \}$ 
 parameterized by $\boldsymbol{\theta} = \{\theta_1,..., \theta_N \}$.
 % and let  be the set of all agents . 
 To implement decentralized control for the proposed EV network, we consider two MARL variants: decentralized critic vs. centralized critic.

\textbf{Decentralized Critic Method:} Among the decentralized policy gradient variants, we first consider the Independent Deep Deterministic Policy Gradient (I-DDPG), where each agent $i$ trains the decentralized policy $\mu_i(a_i\mid o_i)$ and the critic $Q_i^{\mu}(o_i, a_i)$. In this method, each agent has an actor-critic architecture, where both actors and critics are trained based on local observations of the agent. The decentralized critic policy gradient can be derived as follows:
\begin{align}\label{IDDPG_return}
    \nabla_{\theta_i} J_d(\theta_i^\mu) = \mathbb{E}_{o_i,a_i \sim \mathcal{D}} [\nabla_{\theta_i}\mu_i(a_i|o_i)\nabla_{a_i}Q_i^\mu (o_i, a_i)|_ {a_i=\mu_i(o_i)}],
\end{align}
where $o_i$ and $a_i$ are the observations and action of agent $i$ sampled from replay buffer $\mathcal{D}$. The policy $\mu_i$ and the critic $Q_i^\mu$ are approximated with the actor and critic deep neural networks, respectively. 

\textbf{Centralized Critic Method:} Similarly, we consider another decentralized multi-agent framework called Centralized Training Decentralized Execution DDPG (CTDE-DDPG). In this method, each agent uses the centralized action-value function $\hat{Q}_i^\mu(\mathbf{o}, a_1,a_2,...,a_N)$ (which is parameterized by $\theta_i^\phi$) in order to update the decentralized policy $\mu_i(a_i \mid o_i)$ (which is  parameterized by $\theta_i^\mu$). The centralized critic estimates the return on the joint observations and actions, which differs from I-DDPG method. Thus, the gradient of expected return in~\eqref{IDDPG_return} will be extended as follows:
\begin{align}\label{CTDE_return}
    \nabla_{\theta_i} J_c(\theta_i^\mu) = \mathbb{E}_{o,a \sim \mathcal{D}} [\nabla_{\theta_i}\mu_i(a_i\mid o_i) \nabla_{a_i}\hat{Q}_i^\mu (o_1,a_1,... \nonumber \\ ,o_N,a_N) |_{a_i=\mu_i(o_i)}].
\end{align}

CTDE-DDPG uses the actions and observations of all agents in the action-value functions  $\hat{Q}_i^\mu$. Furthermore, as the policy of an agent (i.e., $\mu_i$) is only conditioned upon its own private observations, the agents can act in a decentralized manner during execution. Furthermore, it should be noted that since each $\hat{Q}_i^\mu$ is learned separately, agents can have different rewards. For ease of exposition, we drop the dependency notation $\mu$ from $Q_i^\mu$ and $\hat{Q}_i^\mu$, as well as the index $i$ from  $\hat{Q}_i^\mu$ by assuming that all CTDE agents have a similar reward structure. Therefore, hereinafter  ${Q}_i (o_i, a_i)$ and $\hat{Q} (\mathbf{o}, \mathbf{a}_{i})$ refer to decentralized and centralized action-value functions, respectively.  
% Here, the experience replay buffer $\mathcal{D}$ contains the tuple $\langle o, o^\prime, a_1,...,a_N,r_1,...,r_N \rangle$. The centralized action-value function $\hat{Q}_i^\mu$ is updated as:
% \begin{equation}
%     \mathcal{L}(\theta_i) = \mathbb{E}_{o,a,r,o^\prime}[(\hat{Q}_i^\mu(o_1,a_1,...,o_N,a_N)-y)^2],
% \end{equation}

% where $y=r_i + \gamma \hat{Q}_i^{\mu^\prime}(o_1^\prime, a_1^\prime,...,o_N^\prime, a_N^\prime)|_{a_j^\prime = \mu_j^\prime(o_j^\prime)}$, and $\boldsymbol{\mu}^\prime = \{\mu_{\theta_1^\prime},..., \mu_{\theta_N^\prime} \} $ is the set of target policies.
% \vspace{0.1in}
% \textbf{State and Observation set}:

Figure~\ref{DDPG_model} depicts the CTDE-DDPG framework that is composed of a control strategy layer with $N$ agents, where each agent is implemented by the DDPG algorithm. Using the CTDE-DDPG framework, the single agent evaluation network has access to additional information during the centralized offline training stage, such as observations and actions of other EV charging controller agents, in addition to the local observation.
% As illustrated, 
% by using the CTDE-DDPG framework,
% in the centralized offline training stage, in addition to the local observation, extra information, i.e., observations and actions of all EV charger controller agents, is also available to the single agent evaluation network. 
In particular, at any given time step $h$, $\{o_i^h, a_i^h, r_i^h, {o^{\prime}}_i^{h}\}$ is saved in the replay buffer associated with the agent $i$, where ${o^{\prime}}_i^h$ denotes the next time step observation. As shown in Fig.~\ref{DDPG_model}, when updating the parameters of the actor and the critic according to the inputted mini-batch of transitions, the actor chooses an action according to the local observation $o_i^h$, where $a^i = \mu_{i}(o_i^h)$. The actions are criticized by the critic, where $\mathbf{o}_{i,nor}^h$, $\mathbf{a}_{i,nor}^h$, and $\mathbf{o^{\prime}}_{i,nor}^h$  denote the normalized joint action, observation, and next state observation, respectively. 

\begin{figure*}[t]
   \centering
    \includegraphics[scale=0.415, trim = 0.1cm 1cm 0.1cm 1.1cm, clip]{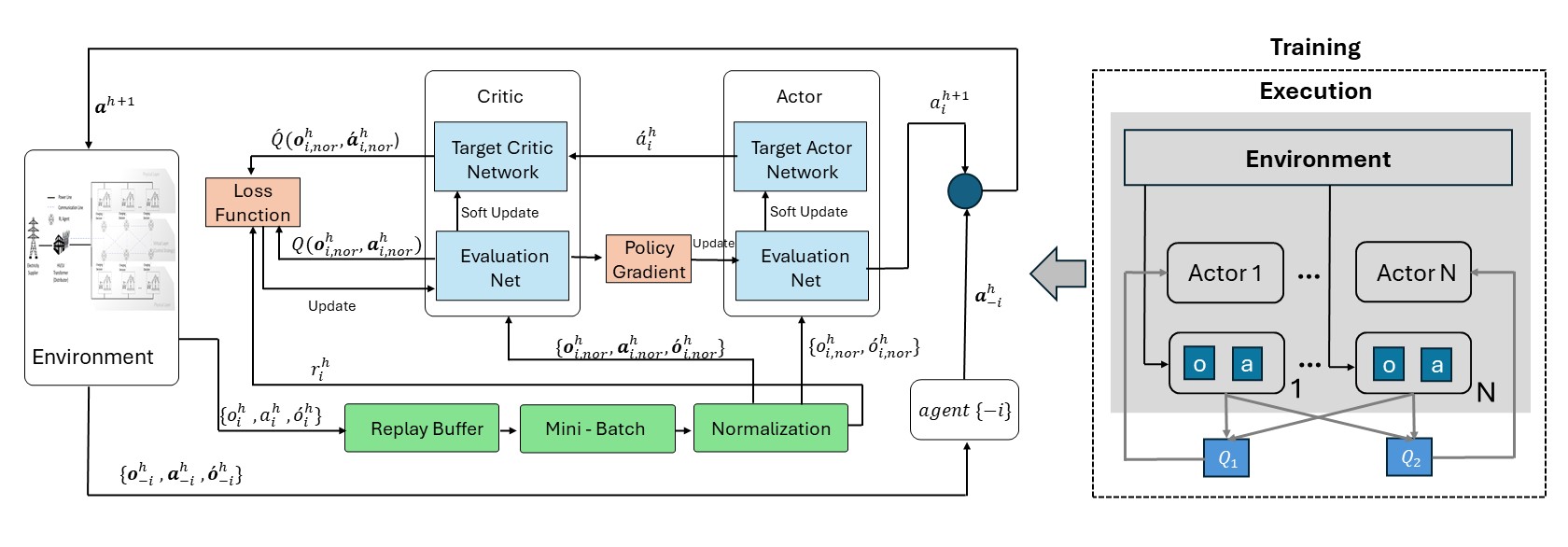}
    \caption{Centralized training decentralized execution (CTDE) multi-agent reinforcement learning framework for EV charging network control.}
    \label{DDPG_model}
    \vspace{-.15in}
\end{figure*}

\vspace{-.3cm}
\subsection{Analysis of Multi-Agent DDPG Methods}
Theoretical analysis for multi-agent settings is critical in order to capture various factors, including nonstationarity that can arise due to the interactions between multiple agents. Dealing with nonstationarity is a significant challenge in MARL because algorithms often assume a stationary environment, where the statistics of the system remain constant. Adopting to nonstationarity requires continuous learning and adjustment of the agents' policies. The nonstationarity can contribute to increased variance in the learning process. Variance can arise from various sources, including stochasticity in the environment, the agents' exploration strategy, and the learning algorithm itself. In MARL, variance analysis becomes more complex due to the interaction and dependencies between multiple agents. The action of an agent can influence the observations and rewards of other agents, leading to increased variance in the learning process. Variance can affect the stability and convergence of learning algorithms. High variance in policy gradient estimates can lead to a larger spread of values, making it challenging to accurately estimate the true gradient. This can result in slower convergence, meaning that more samples may be required to obtain a reliable gradient estimate. 

Several researches have analyzed the variance of {stochastic} policy gradient methods~\cite{weaver2013optimal,konda1999actor,lyu2021contrasting,kuba2021settling}. 
In particular,~\cite{weaver2013optimal} designed an advantage function using a baseline without adding any additional bias to the gradient. In~\cite{konda1999actor}, the authors provide a temporal difference (TD) error as an unbiased estimate of the advantage function.
In other works,~\cite{lyu2021contrasting,kuba2021settling} compared centralized and decentralized frameworks in terms of bias and variance.
% and demonstrated that after an infinite amount of training, the two methods converge to the same policy gradients. 
All of the aforementioned research has focused on analyzing the variance of \emph{stochastic} policy gradients, while less attention has been paid to the theoretical analysis of \emph{deterministic} policies.

In this section, we provide a detailed analysis of the expectation and variance of the policy gradients for the proposed I-DDPG and CTDE-DDPG frameworks.  We first show that the gradient updates in \eqref{IDDPG_return} and \eqref{CTDE_return} are the same in expectation. Next, we prove that the variance of the policy gradient in \eqref{CTDE_return} is at least as large as the variance of I-DDPG. 
% Next, we present a detailed analysis of the convergence and variance for the I-DDPG and CTDE-DDPG frameworks. 
To this end,  we rely on the following assumptions.
\begin{assumption}
The state space $\mathcal{S}$ is either discrete and finite, or continuous and compact. 
\end{assumption}
\begin{assumption}
Every agent's action space $\mathcal{A}_i$ is continuous and compact.  
\end{assumption}
\begin{assumption}
For any agent $i$, state $s \in \mathcal{S}$, and action $a_i \in \mathcal{A}_i$, the mapping $\theta_i \rightarrow \mu_{i}(a_i|o_i)$, $Q_i(o_i,a_i)$, and $\hat{Q}(\mathbf{o},\mathbf{a})$ are continuously differentiable.    
\end{assumption}
\begin{lem}\label{Lemma1}
After the convergence of the critic network for CTDE-DDPG and I-DDPG, the following equality holds:
\vspace{-0.05in}
\begin{align}
\nabla_{a_i}Q_i(o_{i}, a_{i})  =  \mathbb{E}_{a_{-i}, o_{-i} \sim \mathcal{D}} \left[\nabla_{a_i}\hat{Q}(\mathbf{o}, a_i,\mathbf{a}_{-i}) \right].
\end{align}
\end{lem}
\begin{proof}
% First, we note that $\mathbf{a}_{-i}$ denotes the joint action of all agents excluding agent $i$. 
From Lemma 1 of \cite{lyu2023centralized}, value function $Q_i(o_{i}, a_{i})$ and $\hat{Q}(\mathbf{o}, a_i,\mathbf{a}_{-i})$ are related to each other as follows:
% \vspace{-0.25in}
\begin{align}
Q_i(o_{i}, a_{i})  = \mathbb{E}_{a_{-i}, o_{-i} \sim \mathcal{D}} \left[\hat{Q}(\mathbf{o}, a_i,\mathbf{a}_{-i}) \right], \nonumber
\end{align}
where $\mathbf{a}_{-i}$ denotes the joint action of all agents except agent $i$.
By taking derivative over the agent $i^{th}$ action and considering the dominated convergence theorem~\cite{avigad2012algorithmic} we have:
\vspace{-0.05in}
\begin{align}
\nabla_{a_i} Q_i(o_{i}, a_{i})  =\nabla_{a_i} \mathbb{E}_{a_{-i}, o_{-i} \sim \mathcal{D}} \left[\hat{Q}(\mathbf{o}, a_i,\mathbf{a}_{-i}) \right] \nonumber \nonumber \\
= \mathbb{E}_{a_{-i}, o_{-i} \sim \mathcal{D}} \left[\nabla_{a_i} \hat{Q}(\mathbf{o}, a_i,\mathbf{a}_{-i}) \right], \nonumber
\end{align}
% \vspace{-0.15in}
which completes the proof.
\end{proof}

In \eqref{IDDPG_return} and \eqref{CTDE_return}, the difference between the I-DDPG and CTDE-DDPG gradient calculations lies in their respective uses of $Q^i(o_{i}, a_{i})$  and $\hat{Q}(\mathbf{o}, a_i,\mathbf{a}_{-i})$. I-DDPG's reliance is on the random variables $o_i$ and $a_i$ alone, while CTDE-DDPG also takes into account additional random variables $\mathbf{a}_{-i}$ and $\mathbf{o}_{-i}$. This implies that within the CTDE-DDPG schema, agents are required to account for the actions of their peers as well as their own. The expected value of the value function mirrors the collective average of all potential joint actions within the environment. Therefore, according to Lemma~\ref{Lemma1}, the decentralized value function converges to reflect the marginal expectation of the centralized value function. Thus, using Lemma~\ref{Lemma1} we have the following theorem.
\begin{thm}\label{Theorem1}
After convergence of the critic network, the CTDE-DDPG and I-DDPG policy gradients are equal in expectation.   
\end{thm}
\begin{proof}
Inspired by~\cite{lyu2021contrasting} and from Lemma~\ref{Lemma1}, the decentralized value function becomes a marginal expectation of the centralized value function after convergence. Thus, substituting Lemma~\ref{Lemma1} in \eqref{IDDPG_return}, we have:
\begin{align}
    &\nabla_{\theta_i} J_d(\theta_i^\mu) = \mathbb{E}_{o_i,a_i \sim \mathcal{D}} \left[\nabla_{\theta_i}\mu_i(a_i|o_i)\nabla_{a_i}Q_i(o_i,a_i)\right] &&\nonumber \\
    &= \mathbb{E}_{o_i,a_i \sim \mathcal{D}} \left[\nabla_{\theta_i}\mu_i(a_i|o_i)\mathbb{E}_{a_{-i}, o_{-i} \sim \mathcal{D}} \left[\nabla_{a_i} \hat{Q}(\mathbf{o}, a_i,\mathbf{a}_{-i}) \right]\right] &&\nonumber \\
    &= \mathbb{E}_{o,a \sim \mathcal{D}} \left[\nabla_{\theta_i}\mu_i(a_i|o_i)\nabla_{a_i} \left[\hat{Q}(\mathbf{o}, a_i,\mathbf{a}_{-i}) \right]\right]  &&\nonumber \\
    &= \mathbb{E}_{o,a \sim \mathcal{D}} \left[\nabla_{\theta_i}\mu_i(a_i|o_i)\nabla_{a_i} \hat{Q}(\mathbf{o}, \mathbf{a})\right] &&\nonumber \\
    &= \nabla_{\theta_i} J_c(\theta_i^\mu), \nonumber
\end{align}
which illustrates the policy gradients of CTDE-DDPG and I-DDPG are equal in expectation.
\end{proof}

Theorem~\ref{Theorem1} implies that once the critic networks have converged, the expected gradients of the actors in both I-DDPG and CTDE-DDPG are identical. This shows that on average, the suggested policy improvements from both algorithms are unbiased and equivalent. In essence, 
neither algorithm consistently outperforms the other in terms of the expected policy gradients. This implies that in terms of expectation, the performance of one method does not always dominate that of the other method. Our numerical evaluation for the EV network confirms this observation. 
% showing that both methods converge to the sub-optimal solution. 
In the following theorem, we investigate the policy gradient variances of CTDE-DDPG and I-DDPG and show that the variance of CTDE is greater than the variance of I-DDPG.
\begin{thm}\label{Theorem2}
After the convergence of the critic networks, the variance of the policy gradient of CTDE-DDPG is greater than that of the I-DDPG framework. 
\end{thm}

\begin{proof}
We start the proof by redefining \eqref{IDDPG_return} and \eqref{CTDE_return} as follows:
\begin{align}
&{\mathbf{g}_{d,i}} = \nabla_{\theta_i} \mu_i(a_{i}|o_{i}) \nabla_{a_i}Q_i(o_{i}, a_{i}),   \nonumber \\
&{\mathbf{g}_{c,i}} = \nabla_{\theta_i} \mu_i(a_{i}|o_{i}) \nabla_{a_i}\hat{Q}(\mathbf{o}, a_i, \mathbf{a}_{-i}). \nonumber
\end{align}
Given Theorem 1, we know that $\mathbf{g}_{d,i}$ and $\mathbf{g}_{c,i}$ have the same expectation as $\zeta = \mathbb{E}\left[\mathbf{g}_{d,i} \right] = \mathbb{E}\left[\mathbf{g}_{c,i} \right]$. Using the variance definition we have:
\begin{flalign}\label{Variancelaw}
&\mathbf{Var}_{o,a \sim \mathcal{D}} \left[{\mathbf{g}_{c,i}}\right] - \mathbf{Var}_{o_i,a_i \sim \mathcal{D}} \left[{\mathbf{g}_{d,i}}\right] &&\nonumber \\
&= \left(\mathbb{E}_{o,a \sim \mathcal{D}} \left[{\mathbf{g}_{c,i} \mathbf{g}_{c,i}
^T} \right] - \zeta\zeta^T \right) -  \left(\mathbb{E}_{o_i,a_i \sim \mathcal{D}} \left[{\mathbf{g}_{d,i} \mathbf{g}_{d,i}
^T} \right] - \zeta\zeta^T \right) &&\nonumber \\ 
&=  \left(\mathbb{E}_{o,a \sim \mathcal{D}} \left[{\mathbf{g}_{c,i} \mathbf{g}_{c,i}
^T} \right] \right) -  \left(\mathbb{E}_{o_i,a_i \sim \mathcal{D}} \left[{\mathbf{g}_{d,i} \mathbf{g}_{d,i}
^T} \right] \right) &&\nonumber \\
&=  \left(\mathbb{E}_{o,a \sim \mathcal{D}} \left[\left( \nabla_{\theta_i} \mu_i(a_{i} | o_i) \right)\left(\nabla_{\theta_i} \mu_i(a_{i} | o_i) \right)^T \left|\left|\nabla_{a_i}\hat{Q}(\mathbf{o}, a_i,\mathbf{a}_{-i})\right|\right|^2  \right] \right) &&\nonumber \\
& -  \left(\mathbb{E}_{o_i,a_i \sim \mathcal{D}} \left[\left( \nabla_{\theta_i} \mu_i(a_{i} | o_i) \right)\left(\nabla_{\theta_i} \mu_i(a_{i} | o_i) \right)^T \left|\left|\nabla_{a_i}Q_i(o_{i}, a_{i})\right|\right|^2 \right] \right). \nonumber &&
\end{flalign}
We define $A_i = \left( \nabla_{\theta_i} \mu_i(a_{i}|o_i) \right)\left(\nabla_{\theta_i} \mu_i(a_{i}|o_i) \right)^T $. Now, considering Lemma 1 we have:
\begin{flalign}
&\mathbf{Var}_{o,a \sim \mathcal{D}} \left[\mathbf{g}_{c,i}\right] - \mathbf{Var}_{o_i,a_i \sim \mathcal{D}} \left[\mathbf{g}_{d,i}\right] && \nonumber \\
&\quad =  \mathbb{E}_{o,a \sim \mathcal{D}} \left[\left|\left|\nabla_{a_i}\hat{Q}(\mathbf{o}, a_i,\mathbf{a}_{-i})\right|\right|^2 A_i \right] && \nonumber \\
&\quad \quad \quad -  \left(\mathbb{E}_{o_i,a_i \sim \mathcal{D}} \left[\left|\left|\nabla_{a_i}Q_i(o_{i}, a_{i})\right|\right|^2 A_i \right] \right) \nonumber \\
&\quad =  \mathbb{E}_{o_i,a_i \sim \mathcal{D}} \left[\mathbb{E}_{o_{-i},a_{-i} \sim \mathcal{D}}\left[\left|\left|\nabla_{a_i}\hat{Q}(\mathbf{o}, a_i,\mathbf{a}_{-i})\right|\right|^2 A_i \right]\right] && \nonumber \\
&\quad \quad \quad -  \left(\mathbb{E}_{o_i,a_i \sim \mathcal{D}} \left[\left|\left|\nabla_{a_i}Q_i(o_{i}, a_{i})\right|\right|^2 A_i \right] \right) \nonumber \\
& \quad = \mathbb{E}_{o_i,a_i \sim \mathcal{D}}\left[A_i\left(\mathbb{E}_{o_{-i},a_{-i} \sim \mathcal{D}}\left[\left|\left|\nabla_{a_i}\hat{Q} (\mathbf{o}, a_i,\mathbf{a}_{-i})\right|\right|^2 \right] \right. \right. && \nonumber \\ 
&\qquad \qquad \qquad \qquad - \left. \left. \left|\left|\mathbb{E}_{o_{-i},a_{-i} \sim \mathcal{D}}\left[\nabla_{a_i}\hat{Q} (\mathbf{o}, a_i,\mathbf{a}_{-i})\right] \right|\right|^2 \right) \right] && \nonumber \\
&\quad = \mathbb{E}_{o_i,a_i \sim \mathcal{D}}\left[A_i \sum\limits_{j=1}^{K}{\mathbf{Var}_j \left(\nabla_{a_i}\hat{Q} (\mathbf{o}, a_i,\mathbf{a}_{-i})\right)}  \right] && \nonumber \\
&\quad \leq \mathbb{E}_{o_i,a_i \sim \mathcal{D}}\left[B \sum\limits_{j=1}^{K}{\mathbf{Var}_j \left(\nabla_{a_i}\hat{Q} (\mathbf{o}, a_i,\mathbf{a}_{-i})\right)},  \right] &&
\end{flalign} 
where 
% $B_i = \sup_{\rm o,\mathbf{a}}{\| \nabla_{\theta_i} \mu_i(a_{i};o_{i}) \|}$ 
$B = \max\limits_{1 \leq j \leq K} \|\nabla_{\theta_i} [\mu_i(a_{i}|o_i)]_j\|^2$ is the upper-bound of the gradient along the $j$-th dimension of the action space, where $K$ is the dimension of the action space. 
\end{proof}

% This Theorem illustrates that a CTDE learner's gradient estimator incurs additional variance due to exploration by other agents. To elaborate, as the value function converges, it becomes evident that the CTDE-DDPG framework exhibits a higher degree of variance compared to the I-DDPG framework. This increased variance can be attributed to the interplay between multiple agents, each exploring the environment to learn and adapt. The presence of these diverse exploration strategies between agents contributes to the additional variance observed in the CTDE-DDPG approach. Theorem 2 offers a better understanding of the complexities associated with such methods, underlining the importance of careful planning and management of multi-agent systems. Therefore, understanding this phenomenon is essential when considering the trade-offs and complexities associated with MARL in nonstationary environments. 

% In the context of MARL, bootstrapping denotes the procedure of estimating the return (or value) associated with a specific state or a state-action pair. This estimation relies on the agent's own interactions within the environment, as well as the interactions of other agents within the same environment. In scenarios characterized by cooperative MARL, an agent's value function can undergo updates by considering the actions and rewards generated by other agents. This shared information plays a pivotal role in enhancing the stability of bootstrapping, as it allows agents to exchange more accurate estimations of forthcoming rewards and state values. 
This theorem illustrates that a CTDE learner's gradient estimator incurs additional variance due to exploration by other agents. To elaborate, as the value function converges, it becomes evident that the CTDE-DDPG framework exhibits a higher degree of variance compared to I-DDPG framework. This increased variance can be attributed to the interplay between multiple agents, each exploring the environment to learn.

Despite the higher policy gradient variance in the centralized critic setup, all agents share a common value function. This shared value function fosters more consistent and cohesive learning performance, as it benefits from the collective experiences of all agents.
This characteristic helps CTDE to mitigate the issues of nonstationarity encountered by decentralized critics, thereby leading to a more stable and reliable learning process.
On the other hand, even though the I-DDPG method has a smaller policy gradient variance, it results in less stable learning targets, espeically as the number of agents increases~\cite{lyu2023centralized}.
% bootstrapping is employed to estimate the value of a given state based not only on its immediate reward, but also on the anticipated value of subsequent states of other agents. agents. 
% Despite potentially higher variance in its policy gradient CTDE offers lower variance for its update signal used for bootstrapping~\cite{lyu2023centralized}. 
 Therefore, considering learning stability and policy gradient variance, a trade-off exists within MARL frameworks, underscoring the importance of careful planning and management of such systems. Our numerical results in Section~\ref{Section:Results} confirm that despite the higher variances in the policy gradient and convergence complexity, the CTDE method provides performance gains due to its cooperative nature.

% &\quad = \mathbb{E}_{o_i,a_i \sim \mathcal{D}} \left[\mathbb{E}_{o_{-i},a_{-i} \sim \mathcal{D}}\left[\left(\nabla_{a_i}\hat{Q}(a_i,\mathbf{a}_{-i},\mathbf{o})\right)^2 \right] \right. &&\nonumber \\
% &\qquad \quad - \left. \left(\mathbb{E}_{o_{-i},a_{-i} \sim \mathcal{D}} \left[\nabla_{a_i}\hat{Q}\left(a_i,\mathbf{a}_{-i},\mathbf{o} \right) \right] \right)^2 A_i \right] &&\nonumber \\
% &\quad = \mathbb{E}_{o_i,a_i \sim \mathcal{D}} \left[\mathbf{Var}_{o_{-i},a_{-i} \sim \mathcal{D}} \left[\nabla_{a_i}\hat{Q}(a_i,\mathbf{a}_{-i},\mathbf{o}) \right] A_i \right] &&\nonumber \\
% &\quad \leq B_i^2 \mathbb{E}_{o_i,a_i \sim \mathcal{D}} \left[\mathbf{Var}_{o_{-i},a_{-i} \sim \mathcal{D}} \left[\nabla_{a_i}\hat{Q}(a_i,\mathbf{a}_{-i},\mathbf{o}) \right] \right] &&\nonumber

%% file: Chapters/Numerical_Result.tex
\vspace{-.3cm}
\section{Numerical Results}\label{Section:Results}
\begin{figure}[t]
   \centering
    \includegraphics[scale=0.35, trim = 0.35cm 0.25cm 0.2cm 0.3cm, clip]{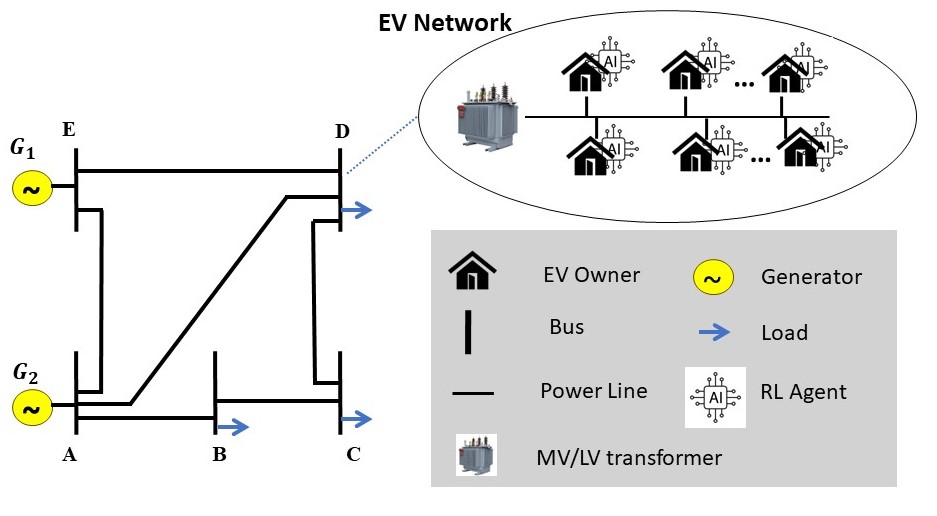}
    \caption{IEEE 5-bus testbed system. An example of how the proposed EV charging network can be integrated into a distribution system.}
    \label{fig:system_model_result}
    \vspace{-.02in}
\end{figure}
In this section, we present comprehensive numerical results to compare the performance of  CTDE-DDPG and I-DDEP methods for EV charging control. 
% CTDE-DDPG, with a decentralized critic charging framework (I-DDPG). 
% To elucidate the effectiveness of our approach, we examine a non-cooperative MA-DDPG framework, where each agent solely observes local data without direct information exchange even throughout training, referred to as I-DDPG. 
% Here, we show that CTDE-DDPG enhances the performance of the I-DDPG algorithm. 
First, we describe the experimental setup, followed by illustrating the impacts of cooperative value function learning. Next, we present our results on  convergence, scalability, and robustness of both frameworks.
\begin{table}[!t]
\caption{HYPERPARAMETERS}
\resizebox{\columnwidth}{!}{
\centering
  {
    \begin{tabular}{l|l|l}
      \hline
      \textbf{Hyperparameters} & \textbf{Centralized Critic} & \textbf{Decentralized Critic}\\
      \hline
      Batch size & 100 & 100 \\
      Discount factor & ${\gamma}$=0.95 & $\gamma$=0.95 \\
      Actor/Critic Optimizer & Adam/Adam & Adam/Adam\\
      Actor Learning Rate/Weight-decay & 0.003/0.0001 & 0.005/0.0003\\
      Critic Learning Rate/Weight-decay & 0.001/0.0001 & 0.001/0.0001\\
      Target Smoothing & $\tau$=0.005 & $\tau$=0.005 \\
      Actor Layers/Nodes & 4/[100,150,100,1] & 4/[100,150,100,1] \\
      Critic Layers/Nodes & 4/[150,200,150,1] & 4/[150,200,150,1] \\
      Actor Activation Functions & [leaky-relu,leaky-relu,leaky-relu,sigmoid] & [rrelu,rrelu,rrelu,sigmoid] \\
      Critic Activation Functions & [leaky-relu,leaky-relu,leaky-relu,linear] & [leaky-relu,leaky-relu,leaky-relu,linear] \\
      Reply Buffer Size & 1000000 & 1000000 \\
      Training Noise & Normal with Decreasing std & Normal with Decreasing std \\
      \hline
      \hline
    \end{tabular}
    }}
% \caption{HYPERPARAMETERS}
 \label{tab:SimulationParam}
 \vspace{-0.1in}
\end{table}
\begin{figure*}[!t]
   \centering
    \includegraphics[scale=0.48, trim = 0.2cm 0.2cm 0.2cm 0.3cm, clip]{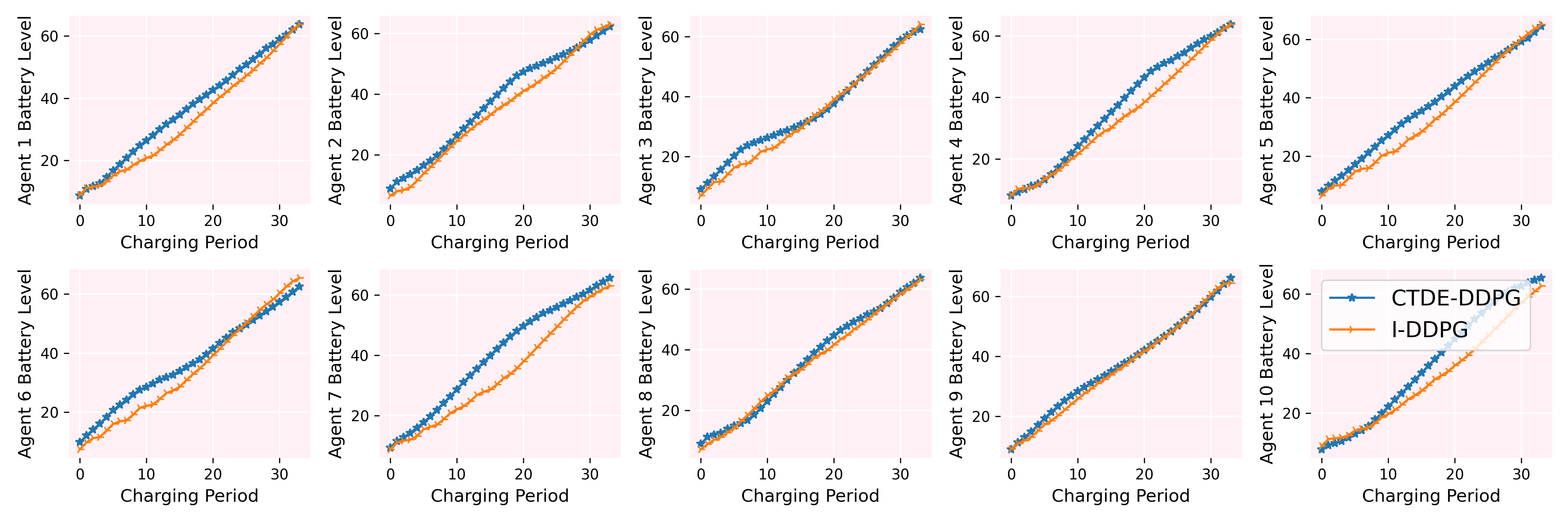}
    \caption{Performance comparison for CTDE-DDPG and I-DDPG frameworks in terms of average battery level during the charging period for 10 agent scenarios.}
    \label{fig:avg_batt_10agents}
    \vspace{-.1in}
\end{figure*}
\begin{figure*}[!t]
   \centering
    \includegraphics[scale=0.48, trim = 0.2cm 0.2cm 0.2cm 0.2cm, clip]{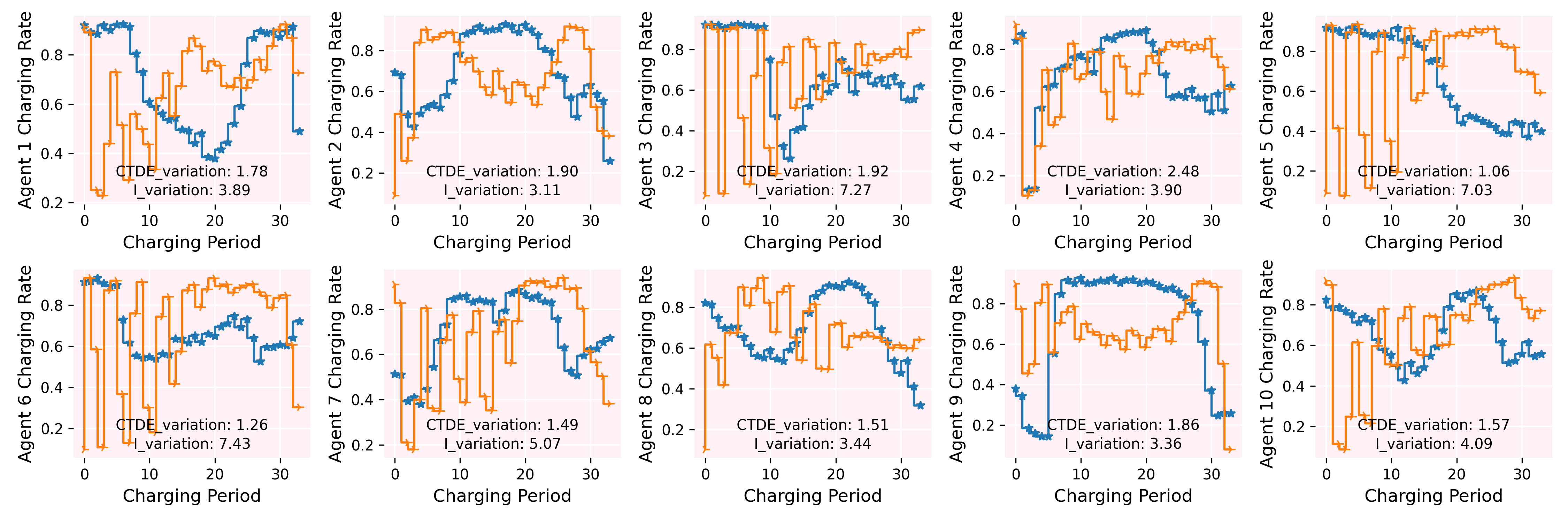}
    \caption{Charging behavior comparison between CTDE-DDPG and I-DDPG frameworks in terms of average charging rate over the charging period for 10 agent scenarios, where one denotes the full rate charging.}
    \label{fig:avg_action_10agents}
    \vspace{-.1in}
\end{figure*}
\vspace{-0.2cm}
\subsection{Experimental Setting}
\noindent 
% \textbf{System Model Setup.}
To assess the performance of our proposed EV charging control, we conducted simulations based on the system model depicted in Fig.~\ref{System Model}. The system model under discussion is designed for compatibility with all IEEE-compliant active distribution system, operating within the framework of distributed locational marginal pricing (DLMP)~\cite{9262887}. As an illustration, we refer to the IEEE 5-bus system shown in Fig.~\ref{fig:system_model_result}, where the integration of the EV network takes place at bus D. This integration leverages the DLMP scheme, whereby each bus within the system is allocated a distinct electricity pricing structure, influenced by local demand dynamics. Specifically, the pricing regime at bus D is closely connected to the demand attributes of the network segments connected to this bus. To facilitate the necessary adjustments in voltage levels, a Medium Voltage/Low Voltage (MV/LV) transformer is linked to bus D. This connection guarantees that the downstream voltage requirements are met adequately. Each EV owner is equipped with an EV charging controller and a smart meter integrated with an RL agent. The main goal of each agent is to maximize its individual learning reward, which is to minimize charging costs and satisfy charging constraints. 
% \noindent 
% \textbf{Simulation Setup.}
In our simulations, we 
% compare centralized and decentralized DDPG frameworks for EV network charging control. 
% We 
implemented the centralized and decentralized framework using Python3 with PyTorch v2.1.0. All simulations were performed via episodic updating across $10,000$ episodes, each of which represents a charging cycle. A cycle consists of 34 iterations. The hyperparameters and simulation setups used  are listed in Table~\ref{tab:SimulationParam}.

\begin{figure*}[!t]
   \centering
    \includegraphics[scale=0.36, trim = 0.1cm 0.2cm 0.2cm 0.2cm, clip]{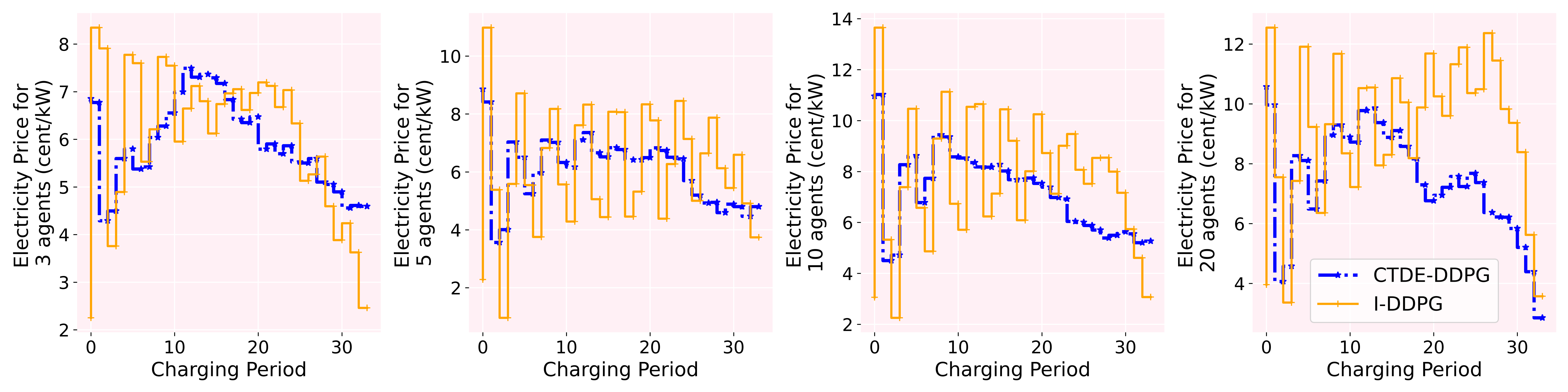}
    \caption{The average electricity price of CTDE-DDPG and I-DDPG for 3, 5, 10, and 20 agents scenarios. }
    \label{fig:avg_price}
    \vspace{-.15in}
\end{figure*}
\begin{figure*}[!t]
   \centering
    \includegraphics[scale=0.36, trim = 0.1cm 0.2cm 0.2cm 0.12cm, clip]{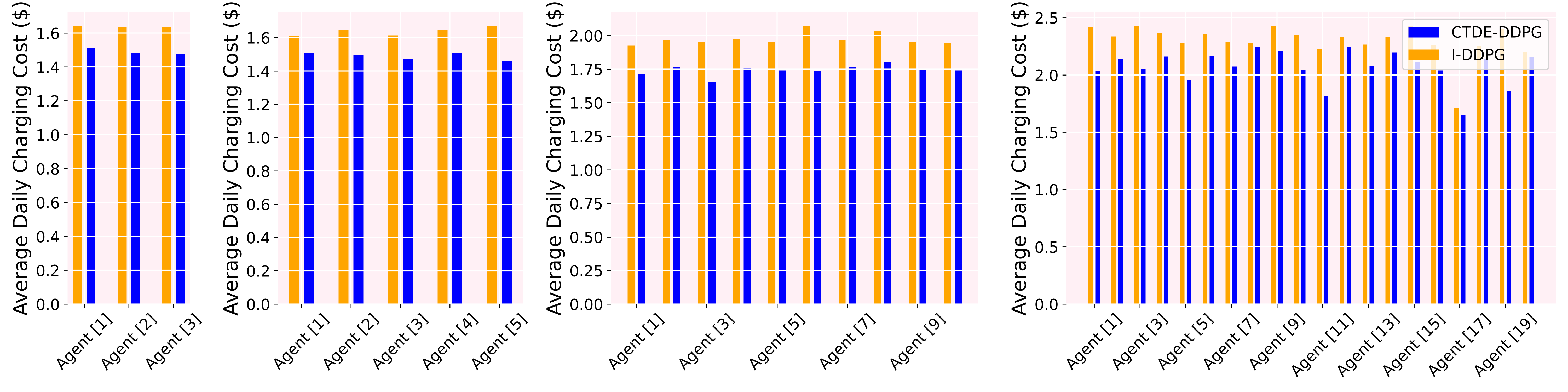}
    \caption{The average charging cost of CTDE-DDPG and I-DDPG for 3, 5, 10, and 20 agents scenarios. }
    \label{fig:avg_cost}
    \vspace{-.15in}
\end{figure*}
\vspace{-0.1cm}
\subsection{MARL General Performance}
In the following, we investigate the general performance of the proposed CTDE-DDPG framework and compare it with I-DDPG. It is worth mentioning that in both frameworks, the primary goal of each individual agent is to fully charge the EVs at the end of the charging phase to meet the demand of the EV owner. Since each individual agent seeks to minimize their charging costs (i.e., maximize the learning return) under dynamic pricing, each agent observes the price signal as a feedback from the environment. To perform the general performance analysis,  we compare the average state-of-the-charge of the batteries over the charging phase in the case of 10 agents. This refers to 10 households with EV charging control, forming a network connected to bus D. To better compare both algorithms, the battery capacity for all 10 agents is considered to be 60kWh in both scenarios. As illustrated in Fig.~\ref{fig:avg_batt_10agents}, both algorithms effectively meet the users' demands for battery charging. This demonstrates successful charging control of EVs within the designated phase by both algorithms.   

\subsection{Cooperative vs Independent Value Function Learning}
For a more comprehensive comparison of the two frameworks, this section delves into the charging patterns exhibited by both algorithms. Additionally, we explore the influence of the number of agents in our system model on the efficacy of the control strategy. Figure~\ref{fig:avg_action_10agents} illustrates the average charging rates of CTDE-DDPG and I-DDPG over the charging period for a scenario with $10$ agents. Utilizing average charging rates enables a more meaningful performance comparison, considering that a single charging cycle may not fully represent the charging behavior of the agent due to the stochastic nature of the environment. Thus, after convergence of the critic network, we execute the EV charging control for 100 more episodes and calculate the average charging rate.
% \begin{equation}\label{eq:AvgAction}
% \Vec{a_i} = \sum\limits_{h \in \mathcal{H}} \frac{1}{q} \sum\limits_{j=1}^{q} a_{i,h}^j
% \end{equation}
% where $\Vec{a_i}$ denotes the vector of average charging rate over the charging phase for the agent $i$, $q$ denotes the number of episodes, and $a_{i,h}^j$ denotes the charging rate of the agent $i$ at time $h$ in episode $j$. Considering this, 
% Figure~\ref{fig:avg_action_10agents} illustrates the comparison of the average charging rate for the two algorithms in a scenario with 10 agents. 
As shown in Fig.~\ref{fig:avg_action_10agents}, in the I-DDPG scenario, the charging behavior exhibits a fluctuating rate, while CTDE-DDPG shows a consistently smooth charging rate throughout the charging phase. To better compare the two scenarios, we define the total variation (denoted by $TV$) for the agnet $i$ and over the charging time $H$ as follows:
\begin{equation}
\vspace{-0.05cm}
    TV_i(H) = \frac{1}{a^\text{max}} \sum\limits_{h =1}^{H}{\mid a^h_i - a_i^{h-1}} \mid,
\end{equation}
where $H$ denotes the charging duration, $a_i^h$ denotes the agent $i$ charging power at time step $h$ and $a^{\text{max}}$ denotes the maximum charging power allowance.
The total variation in charging behavior during the charging phase is higher in I-DDPG compared to CTDE-DDPG, potentially leading to a degradation in battery lifetime for I-DDPG. Figure~\ref{fig:avg_action_10agents} demonstrates at least a 36$\%$ reduction in the total variation of the average charging rate using CTDE version. The smooth charging pattern exhibited by the CTDE version suggests that the proposed cooperative MARL surpasses the independent MARL version for EV network charging control. 
This observation is further supported when we compare the impact of agents' charging behavior on the charging cost. 

In particular, Fig.~\ref{fig:avg_price} provides a comparative analysis of the average electricity price of the network under both algorithms. The results illustrate that with an increasing number of agents, the disparity in the average electricity price between the two algorithms becomes more pronounced. Additionally, the fluctuations in charging behavior within the I-DDPG scenario lead to corresponding fluctuations in electricity pricing.
% , which is not desirable for EV chargers. 
% , which causes financial losses . 
In contrast, the CTDE framework exhibits a more consistent electricity price during the charging phase. This consistency underscores the economic advantages of cooperative behavior among agents, highlighting the efficiency of the CTDE approach in maintaining price stability in the EV charging network.  This price consistency, in addition to robust charging behavior in CTDE-DDPG, leads to lower daily costs in a cooperative framework. Figure~\ref{fig:avg_cost} depicts the distinction in daily costs between CTDE-DDPG and I-DDPG by showcasing the average daily cost for scenarios with 3, 5, 10, and 20 agents, respectively. As illustrated, in all scenarios, CTDE-DDPG outperforms I-DDPG by reducing the charging cost for all agents. 
\vspace{-0.35cm}
\subsection{Convergence and Fairness Analysis}
In this section, we investigate how increasing the number of agents in our system model impacts the performance of the two algorithms. In Fig.~\ref{fig:avg_return}, we compare both algorithms' average episode returns. This involves calculating the average return and respective variances of the algorithms. 
% as follows:
% \begin{equation}\label{eq:AvgReturn}
% \Bar{\mu}_j = \frac{1}{N}\sum\limits_{n=1}^{N} \frac{1}{P} \sum\limits_{i=1}^{P} Re_j^i \ \ \ \ for \ \ j=1,2,3,...,10000,
% \end{equation}
% where $Re_j^i$ denotes the episode return for episode $j$ in simulation $i$. $P$ denotes the number of simulations with random inputs and $N$ denotes the number of agents. Also for variance, we have:
% \begin{align}\label{eq:AvgVar}
% \sigma_j^2 = \sum\limits_{i=1}^{P} Var(Re_j^i) + 2 \sum\limits_{i < k} Cov(Re_j^i , Re_j^k) \nonumber \\ for \ \ j=1,2,3,...,10000.
% \end{align}
The results illustrate that by increasing the number of agents, both the CTDE and I-DDPG algorithms exhibit convergence to a common policy, reflected in similar points in terms of average return. 
Therefore, the results in Fig.~\ref{fig:avg_return} reveal a shared policy behavior between the two algorithms.
However, it should be noted that the variance of the return also increases. This phenomenon is attributed to the nonstationarity nature of the MARL frameworks. 
% Furthermore, the results indicate that the variance of CTDE becomes comparable to that of I-DDPG as the number of agents increases.
% thereby substantiating our theoretical analysis.
% presented in~\cite{lyu2021contrasting}.

 However, a significant implication of employing CTDE-DDPG is shown in Fig.~\ref{fig:avg_performance} in which we capture the \emph{fairness} performance metric defined as the ratio of the worst-performing agent to the best-performing agent in terms of average return. We calculate the fairness ratio as the number of agents increases.
As shown in Fig.~\ref{fig:avg_performance}, there is a noticeable decline in I-DDPG performance as the number of agents increases. This decline can be attributed to the absence of cooperation between agents in the I-DDPG, where agents do not consider the policies of other agents within the system. This lack of collaboration adversely impacts the overall performance of I-DDPG in multi-agent scenarios, such that some of the agents may perform poorly compared with best-performing agents.

% \begin{figure}[t]
%    \centering
%     \includegraphics[scale=0.29, trim = 3.3cm 0.2cm 0.2cm 0.2cm, clip]{figs/norm.png}
%     \caption{Norm}
%     \vspace{-.1in}
% \end{figure}

\begin{figure}[!t]
   \centering
    \includegraphics[width=\linewidth, trim = 0.2cm 0.2cm 0.2cm 0.2cm, clip]{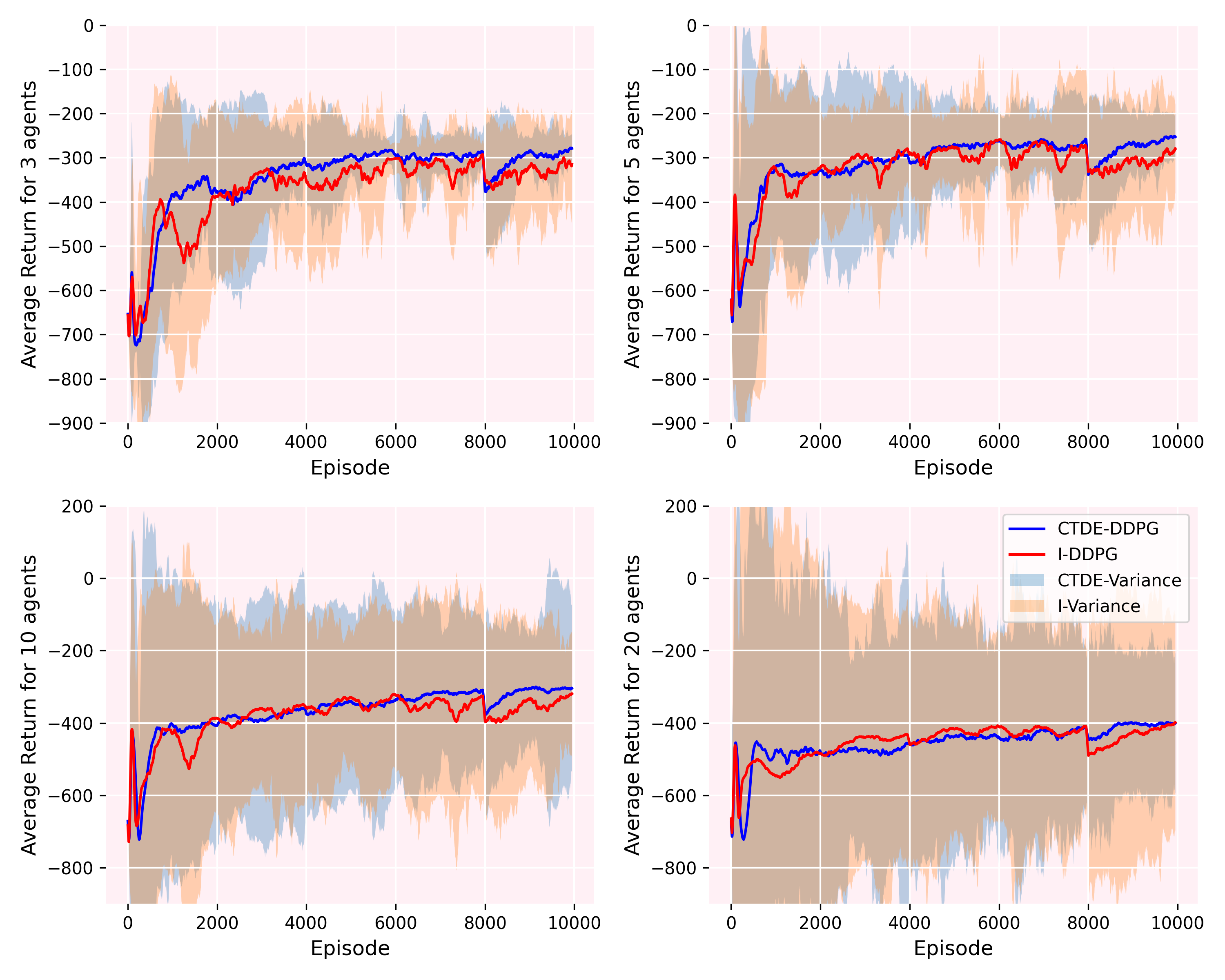}
    \caption{The average episodic reward and its variance for the CTDE-DDPG and I-DDPG methods in 3, 5, 10, and 20 agents scenarios.  }
    \label{fig:avg_return}
    \vspace{-.1in}
\end{figure}
\begin{figure}[!t]
   \centering
    \includegraphics[scale=0.22, trim = 0.2cm 0.2cm 0.2cm 0.2cm, clip]{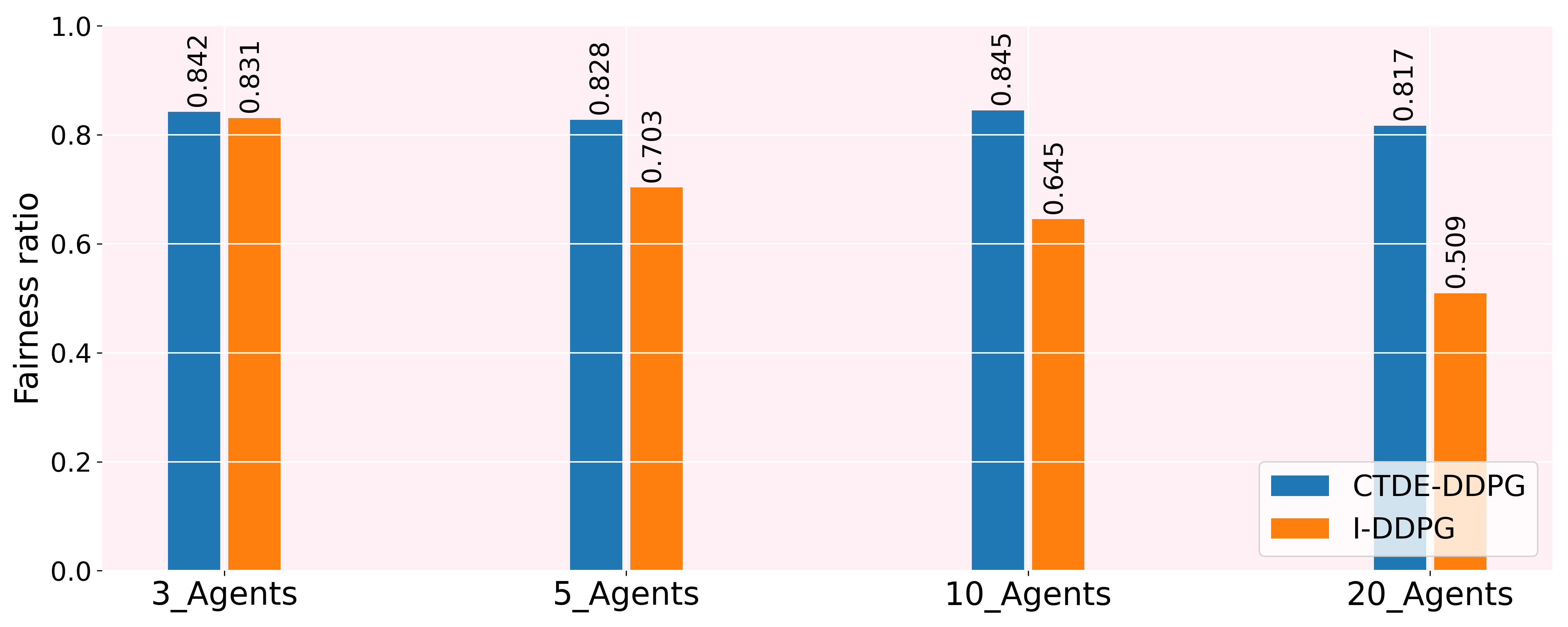}
    \caption{The worse-case to the best-case agents' performance in 3, 5, 10, and 20 agents scenarios.}
    \label{fig:avg_performance}
    \vspace{-.2in}
\end{figure}

%% file: Chapters/Conclusion.tex
\section{Conclusion}\label{Section:Conclusion}
In this paper, we introduced an efficient decentralized framework for EV charging network control. Our approach utilized a centralized training-decentralized execution deep deterministic policy gradient (CTDE-DDPG) reinforcement learning. This framework allows agents to collect additional information from other EVs exclusively during the training phase, while maintaining a fully decentralized strategy during the execution phase. We formulated the charging problem as a decentralized partially observable Markov decision process (Dec-POMDP).  Furthermore, we conducted a comparative analysis between our proposed framework and a baseline approach where independent DDPG (I-DDPG) agents individually solve their local charging problems without any information from other agents, even during the training phase.  We presented a theoretical analysis on the expectation and variance of the policy gradient for the CTDE-DDPG and I-DDPG methods.  
    % that selecting a centralized critic is not universally preferable and relies on the environment, since both CTDE-DDPG and I-DDPG methods converge to the same policy on average. In fact, employing a centralized critic results in higher variance, which adversely affects learning stability and convergence speed. 
    % On the other hand, the CTDE-DDPG algorithm combats nonstationarity due to the cooperation between agents during training. 
Our simulation results demonstrated that with cooperation between agents in CTDE-DDPG, the overall network cost and the average electricity price decrease, leading to reduced individual costs. Furthermore, our results indicate that compared with I-DDPG, CTDE-DDPG achieves a more robust and fair performance as the number of agents increases.
\vspace{-0.2cm}